\documentclass[nohyperref, nonacm]{article}

\usepackage{microtype}
\usepackage{graphicx}
\usepackage{subfigure}
\usepackage{booktabs} 
\usepackage{centernot}
\usepackage{tikz}
\usepackage{enumitem}
\usepackage{xcolor}
\usepackage[accepted]{icml2022}
\usetikzlibrary{arrows,shapes.arrows,shapes.geometric,
	shapes.multipartb,backgrounds,decorations.pathmorphing,positioning,fit,automata}
\tikzset{
	>=stealth',
	true/.style={
		rectangle,
		draw=black, very thick,
		text width=6.5em,
		minimum height=2em,
		text centered,
		fill=gray, opacity = 0.5},
	punkt/.style={
		rectangle,
		rounded corners,
		draw=black, very thick,
		text width=6.5em,
		minimum height=2em,
		text centered},
	est/.style={
		circle,
		draw=black, very thick,
		text centered},
	shade/.style={
		circle,
		draw=black, very thick, fill=gray!50,
		text centered},
	weight/.style={
		circle,
		draw=black, very thick,
		text width=6.5em,
		minimum height=2em,
		text centered},
	pil/.style={
		->,
		thick,
		shorten <=2pt,
		shorten >=2pt,},
	double/.style={
		<->,
		thick,
		shorten <=2pt,
		shorten >=2pt,},
	dash/.style={
		dashed,
		thick,
		shorten <=2pt,
		shorten >=2pt,},
	dashdouble/.style={
		<->,
		dashed,
		thick,
		shorten <=2pt,
		shorten >=2pt,}
}

\usepackage{hyperref}

\usepackage[accepted]{icml2022}

\usepackage{amsmath}
\usepackage{amssymb}
\usepackage{mathtools}
\usepackage{amsthm}

\usepackage[capitalize,noabbrev]{cleveref}

\theoremstyle{plain}
\newtheorem{theorem}{Theorem}[section]
\newtheorem{proposition}[theorem]{Proposition}
\newtheorem{lemma}[theorem]{Lemma}

\theoremstyle{definition}
\newtheorem{definition}[theorem]{Definition}
\newtheorem{assumption}[theorem]{Assumption}
\theoremstyle{remark}
\newtheorem{remark}[theorem]{Remark}
\newtheorem{example}[theorem]{Example}

\newcommand\norm[1]{\left\lVert#1\right\rVert}

\newcommand{\bE}{\mathbb{E}}

\newcommand{\bI}{\mathbb{I}}

\newcommand{\nCI}{\centernot{\CI}}
\newcommand{\CI}{\mathrel{\perp\mspace{-10mu}\perp}}

\DeclareMathOperator*{\argmax}{arg\,max}

\usepackage[textsize=tiny]{todonotes}
\usepackage[suppress]{color-edits}

\addauthor{lg}{blue}
\addauthor{kh}{purple}
\addauthor{sw}{red}
\addauthor{ac}{teal}

\addauthor{ov}{orange}

\icmltitlerunning{Predictive Performance Comparison of Decision Policies Under Confounding}

\begin{document}

\twocolumn[
\icmltitle{Predictive Performance Comparison of Decision Policies Under Confounding}




\begin{icmlauthorlist}
\icmlauthor{Luke Guerdan}{cmu}
\icmlauthor{Amanda Coston}{msr}
\icmlauthor{Kenneth Holstein}{cmu}
\icmlauthor{Zhiwei Steven Wu}{cmu}
\end{icmlauthorlist}

\icmlaffiliation{cmu}{Carnegie Mellon University.}
\icmlaffiliation{msr}{Microsoft Research}

\icmlcorrespondingauthor{Luke Guerdan}{lguerdan@cs.cmu.edu}


\vskip 0.3in
]



\printAffiliationsAndNotice{}  


\begin{abstract}
Predictive models are often introduced to decision-making tasks under the rationale that they improve performance over an existing decision-making policy.  However, it is challenging to compare predictive performance against an existing decision-making policy that is generally under-specified and dependent on unobservable factors. These sources of uncertainty are often addressed in practice by making strong assumptions about the data-generating mechanism.  In this work, we propose a method to compare the predictive performance of decision policies under a variety of modern identification approaches from the causal inference and off-policy evaluation literatures (e.g., instrumental variable, marginal sensitivity model, proximal variable). Key to our method is the insight that there are regions of uncertainty that we can safely ignore in the policy comparison.  We develop a practical approach for finite-sample estimation of regret intervals under no assumptions on the parametric form of the status quo policy. We verify our framework theoretically and via synthetic data experiments. We conclude with a real-world application using our framework to support a pre-deployment evaluation of a proposed modification to a healthcare enrollment policy.
\end{abstract}

\begin{figure}
    \centering
    \includegraphics[trim=0.5cm .25cm 0.4cm .2cm, clip, scale=0.73]{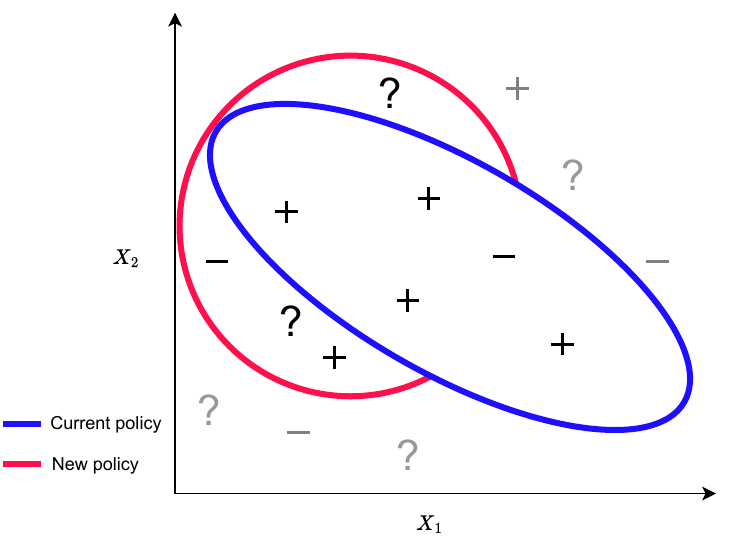}
    \vspace{-5mm}
    \caption{Illustration of uncertainty in comparing two policies in a toy setting with $X \in \mathcal{R}^2$. Points are labelled by their outcome: positive (+), negative (-) or unknown (?). Ovals denote the selection region of a policy. Points that neither policy selects (denoted by grey) are irrelevant to the comparison. Our method leverages this to reduce policy comparison uncertainty. }
    \label{fig:comparison_toy_plot}
    \vspace{-2mm}
\end{figure}

\section{Introduction}

Predictive models are often introduced under the rationale that they improve performance over an existing decision-making policy \citep{grove2000clinical}. For example, models have been developed with the goal of improving the accuracy of human decisions in domains such as healthcare, criminal justice, and education \citep{baker2021algorithmic, rambachan2021identifying}. Given the high-stakes nature of these domains, regulatory frameworks have called for organizations to provide explicit comparisons of predictive models against the status quo they are intended to replace \citep{aaa2023, johnson2022bureaucratic}. For example, consider two common policy comparison settings of real-world significance.

\begin{example}[Human vs algorithm decisions]\label{label:human_alg}
Let the status quo policy consist of human-only decisions (e.g., pre-trial release decisions \citep{kleinberg2018human}, medical testing decisions \citep{mullainathan2019machine}) made using covariates and unobserved contextual information. We would like to evaluate whether a model recommending actions by thresholding risk predictions would constitute an improvement over the current human-only policy (e.g., \citet{rambachan2021identifying, chouldechova2018case}).

\end{example}

\begin{example}[Human$+$algorithm vs algorithm decisions]\label{label:human_alg_alg}
Let the status quo policy consist of humans making decisions with the support of a predictive model, covariates, and unobserved contextual information. We would like to assess whether decisions made using the predictive model alone would yield a performance improvement over the human-algorithm combination (e.g., \citet{cheng2022child}). 
\end{example}

Both of these examples can be formalized as confounded off-policy evaluation in the contextual bandits setting \citep{jung2020bayesian, swaminathan2015counterfactual, tennenholtz2021bandits, rambachan2022counterfactual}. Given observational data collected under a confounded status quo policy, our goal is to compare predictive performance against an alternative policy which would assign different actions. \lgedit{A common strategy for off-policy evaluation under confounding involves \textit{partial identification} of policy performance within an uncertainty interval \citep{kallus2021minimax, rambachan2021identifying, rambachan2022counterfactual, pu2021estimating, zhang2021selecting}. Yet constructing informative performance intervals is challenging because it often requires imposing untestable assumptions on the structure and magnitude of unmeasured confounders impacting the status quo policy. 

In this work, we develop a novel partial identification method which reduces uncertainty in policy performance comparisons by isolating confounding-related uncertainty in the disagreement region of the action space (Figure \ref{fig:comparison_toy_plot}). Our uncertainty cancellation approach yields tighter regret intervals than existing frameworks designed for non-comparative performance evaluations \citep{rambachan2022counterfactual, namkoong2020off}, which account for redundant uncertainty in the agreement region of the action space. Our simple approach operates by (1) decomposing policy performance measures into sets of identified and partially identified comparison statistics, then (2) cancelling common partially identified terms in the regret estimand. 

We show that a diverse set of causal assumptions studied in prior off-policy evaluation literature imply uncertainty sets over partially identified comparison statistics. As a result, while existing off-policy evaluation frameworks are often tied to a fixed causal assumption (e.g., the marginal sensitivity model \citep{tan2006distributional}) that may be unjustified in some contexts, ours is interoperable with a broad family of modern identification strategies (e.g., Rosenbaum's $\Gamma$ model \citep{namkoong2020off, zhang2021selecting}, instrumental variable \citep{lakkaraju2017selective, kleinberg2018human}, proximal variable \citep{ghassami2023partial}, or the marginal sensitivity model (MSM) \citep{kallus2021minimax}).} \lgdelete{However, existing off-policy evaluation frameworks do not support comparison of predictive performance metrics in the settings outlined above. First, existing approaches target the expected potential outcome given actions selected under the new policy, i.e. $V(\pi) = \bE[Y(\pi)]$ \citep{kallus2020confounding, namkoong2020off, zhang2021selecting, uehara2022review, si2020distributionally, kallus2020confounding, hatt2022generalizing, ek2023off, tennenholtz2021bandits, ishikawa2023kernel}. Yet real-world model evaluations require fine-grained assessments of predictive performance metrics, such as the positive and negative class performance. These measures are challenging to estimate from off-policy data because they require conditioning on a partially observed potential outcome. While \citet{rambachan2022counterfactual} develop an evaluation framework targeting predictive performance measures under confounding, their approach does not support comparisons against the status quo data-generating policy. Additionally, prior frameworks for off-policy learning and evaluation rely upon a specific causal assumption (e.g., Rosenbaum's $\Gamma$ model \citep{namkoong2020off, zhang2021selecting}, instrumental variables \citep{lakkaraju2017selective, kleinberg2018human}, or the marginal sensitivity model (MSM) \citep{kallus2021minimax}). In real-world auditing settings, it is critical that evaluation frameworks are compatible with a broad set of causal assumptions to support flexibility. 

In this work, we develop the first framework supporting predictive-performance comparisons of decision-making policies under confounding. Our approach \textit{partially identifies} the performance of a proposed policy against a status quo data-generating policy. We decompose performance measures into sets of identified and partially identified comparison statistics. We show that a diverse set of causal assumptions studied in prior off-policy evaluation literature imply uncertainty sets over partially identified sufficient statistics. We develop an approach for tightly characterizing policy regret bounds by isolating uncertainty relevant to the policy comparison. Our approach leverages the key insight that we can eliminate redundant uncertainty which does not contribute to differential policy performance (Figure \ref{fig:comparison_toy_plot}).} 

We propose a flexible plug-in algorithm for finite sample estimation of regret bounds under no parametric assumptions on the data-generating policy. While this procedure inherits slow non-parametric convergence rates of regression functions used to estimate bounds on comparison statistics, we also show how to construct assumption-tailored doubly robust estimators which attain fast root-$n$ convergence rates under no parametric assumptions. We conduct synthetic experiments verifying the coverage of our regret intervals. We conclude by illustrating how our framework can be used to support a pre-deployment evaluation of a proposed modification to a healthcare enrollment policy \citep{obermeyer2019dissecting}. \lgedit{Our results demonstrate that, in some cases, our improved regret interval supports more conclusive pre-deployment assessments of decision policies than would be possible via existing off-policy evaluation approaches.} 

We make the following main contributions:

\setlist{nolistsep} 
 \begin{itemize}[noitemsep]
    \item We formulate the problem of comparative predictive performance evaluations for decision making policies under unmeasured confounding ($\S$ \ref{sec:preliminaries}).

    \item We propose a novel partial identification technique that yields informative bounds on predictive performance differences by isolating comparison-relevant uncertainty ($\S$ \ref{sec:bound_identification}). Our technique is interoperable with a range of modern identification approaches from causal inference and off-policy evaluation literature ($\S$ \ref{sec:pid_framework}). 

    \item We develop flexible methods for finite sample estimation of regret bounds under no parametric assumptions on the confounded status quo policy ($\S$ \ref{sec:bound_estimation}).
    
    \item We validate our framework theoretically and via synthetic data experiments ($\S$ \ref{sec:numeric_experiments}). We demonstrate how our framework can support a pre-deployment evaluation of a proposed healthcare enrollment policy ($\S$ \ref{sec:realworld_experiments}). We present all proofs in the Appendix.\footnote{All code for experiments is publicly available \href{https://github.com/lguerdan/icml24_predictive_performance_comparison_dps}{here}.}

\end{itemize}

\section{Related Work}\label{sec:related_work}

Off-policy evaluation (OPE) is a widely studied problem in reinforcement learning \citep{uehara2022review, jiang2016doubly, precup2000eligibility, namkoong2020off, kallus2020confounding}, contextual bandits \citep{si2020distributionally, tennenholtz2021bandits, swaminathan2015counterfactual, xu2021deep, dudik2014doubly}, and econometrics \citep{kitagawa2018should, huber2019introduction} literature. \lgdelete{To our knowledge, we offer the first off-policy evaluation framework supporting predictive performance comparisons of decision policies.} We study the confounded offline contextual bandits setting \citep{jung2020bayesian, swaminathan2015counterfactual, tennenholtz2021bandits}, in which observational data is collected under a behavior policy $\pi_0(X_i, U_i)$ which acts on covariates $X_i$ and unmeasured confounders $U_i$.

\lgedit{Existing OPE frameworks typically estimate the \textit{value function} $V(\pi) = \bE[Y(\pi(X))]$ of a new policy $\pi(X_i)$ via observational data \citep{kallus2020confounding, namkoong2020off, zhang2021selecting, uehara2022review, si2020distributionally, kallus2020confounding, hatt2022generalizing, ek2023off, tennenholtz2021bandits, ishikawa2023kernel}.\footnote{Work in the confounded offline setting typically studies updated policies that only observe covariates. When $\pi(X_i)$ also observes $U_i$ regret intervals tend to be vacuous. } In this work, we instead target predictive performance measures (e.g., Accuracy, TNR, PPV), which are often of interest during policy evaluation of algorithmic risk assessments. While OPE of predictive performance measures introduces estimand-specific challenges (i.e., conditioning on a partially-observed potential outcome), approaches developed for OPE of value functions are also applicable in the algorithmic decision-making context (e.g., doubly robust estimation, reweighting).\footnote{In the algorithmic decision-making context, the term \textit{selective labels} is often used to describe a setting in which a potential outcome is only observed under one of the possible actions \citep{lakkaraju2017selective, wei2021decision, coston2021characterizing, de2018learning}. }

In particular, \citet{rambachan2022counterfactual} develop a partial identification framework for robust learning and evaluation of algorithmic risk assessments under confounding. This framework leverages doubly-robust approaches for OPE of predictive performance measures. One natural strategy for extending this framework to the policy comparison setting involves computing bounds on $V(\pi)$ and $V(\pi_0)$ \textit{independently} before computing the regret interval by taking a difference. However, we show that this approach yields an overly conservative regret interval by introducing redundant uncertainty in the agreement region of the policy action space. We introduce an alternative $\delta$-regret interval, which we show yields more informative bounds than this baseline approach. Moreover, we extend the Mean Outcome Sensitivity Model (MOSM) introduced by \citet{rambachan2022counterfactual} to develop uncertainty sets around partially identified policy comparison statistics which are consistent with a range of modern causal assumptions (e.g., instrumental variable, MSM, Rosenbaum's $\Gamma$ sensitivity model).}

Our work also builds upon prior literature studying policy learning and selection under confounding \citep{zhang2021selecting, konyushova2021active, yang2022offline, kuzborskij2021confident, zhang2021selecting, kallus2021minimax, kallus2018confounding, hatt2022generalizing, ek2023off, gao2023confounding, balachandar2023domain}. \citet{zhang2021selecting} develop a framework for ranking individualized treatment policies under Rosenbaum's $\Gamma$ sensitivity model. \citet{namkoong2020off} also use Rosenbaum's $\Gamma$ sensitivity model to partially identify $V(\pi)$ in the sequential setting. \citet{kallus2021minimax, kallus2018confounding} develop an approach for learning minimax optimal decision-policies under the MSM. In principle, confounding-robust policy learning frameworks such as these can also be used to recover $\delta$-regret intervals against the confounded data-generating policy \citep{hatt2022generalizing, ek2023off, gao2023confounding}. However, these techniques (1) target the mean potential outcome rather than predictive performance measures, and (2) are tied to a fixed causal assumption which may not be applicable in all contexts. In contrast, our framework is interoperable with a range of assumptions (e.g., instrumental variable, MSM, proximal variable) studied in prior OPE literature.

\lgedit{Finally, our work relates to literature studying comparison of human versus algorithmic decision-making policies.} \citet{kleinberg2018human} develop a framework for comparing human decisions and algorithmic predictions under selective labels and unobservables. The contraction technique underpinning this approach compares the failure rate of judicial decisions against a predictive model under an assumption that judges have heterogeneous selection rates \citep{lakkaraju2017selective}. Later developments formalize this approach into an instrumental variable framework \citep{chen2023learning, rambachan2021identifying, rambachan2022counterfactual}. \citet{rambachan2021identifying} develop a framework for identifying systematic prediction mistakes in historical human decisions. This framework enables utility-based comparisons of human versus algorithmic decision-making policies under varying sets of econometric assumptions. \lgedit{Most recently, \citet{ben2024does} devise an experimental framework for evaluating human+algorithm hybrid workflows against the quality of decisions that a human or algorithm would make alone. This framework supports \textit{post-deployment} policy comparisons following randomized assignment to a treatment (algorithmic recommendation) or control (no algorithmic recommendation) condition. In contrast, our framework is designed to support \textit{pre-deployment} policy comparisons given observational data collected under a status quo policy.}

\section{Preliminaries}\label{sec:preliminaries}

Let $\pi_0: \mathcal{X} \times \mathcal{U} \rightarrow \mathcal{A}$ be a status quo decision-making policy assigning binary actions given measured covariates $X \in \mathcal{X}$ and unmeasured confounders $U \in \mathcal{U}$. Let $\pi: \mathcal{X} \rightarrow \mathcal{A}$ be a proposed policy that assigns binary actions only via covariates.\footnote{Our framework is compatible with new policies which are stochastic and deterministic. However we assume that $\pi$ is known.} For example, in a medical testing context, $\pi_0$ is an existing (e.g., physician based) testing policy, while $\pi$ is a proposed algorithmic policy. Let $D^{\pi_0}$, $T^{\pi}$ be random variables indicating actions selected under the status quo and proposed policies, respectively.\footnote{We sometimes omit policy superscripts over random variables to ease notation.} Let $Y(1) \in \mathcal{Y}$ be the outcome of interest to the policy comparison. This is the potential outcome which \textit{would be observed} given a positive decision ($D^{\pi_0}=1$) selected under the status quo policy \citep{rubin2005causal}. For example, $Y(1)$ denotes the disease status of a patient, or repayment if a lendee is granted a loan. Let $Y \in \mathcal{Y}$  be the binary outcome observed in observational data. In our selective labels setting \citep{lakkaraju2017selective}, $Y$ is only observed when an instance received a positive decision ($D^{\pi_0}=1$) under $\pi_0$. We make the following standard consistency assumption on $Y$ \citep{rubin2005causal}.

\begin{assumption}[Consistency]\label{assumption:consistency}
$D^{\pi_0}=1 \implies Y = Y(1)$.
\end{assumption}

For example, consistency would be violated if the testing decision for one patient impacts the disease status of another \citep{hudgens2008toward}. Additionally, we assume that each instance has some probability of both decisions under the status quo policy.

\begin{assumption}[Positivity]\label{assumption:positivity}
$p(D^{\pi_0}=a \mid X=x, U=u) > 0, \; \forall a \in \mathcal{A}, x \in \mathcal{X}, u \in \mathcal{U}$.
\end{assumption}

\subsection{Problem Formulation}\label{subsec:problem_formulation}

Given the target distribution $(X, U, D^{\pi_0}, T^{\pi}, Y(1), Y) \sim p^*(\cdot)$, our goal is to evaluate the regret $R^*(\pi, \pi_0; m) = m^*(\pi) - m^*(\pi_0)$, where $m^*: \mathcal{A} \times \mathcal{Y} \rightarrow \mathbb{R}_+$ is a policy performance measure computed with respect to $Y(1)$. We study the following performance measures, which are analogously defined for the baseline policy.

\textbf{Predictive performance:} Let $m^*_y(\pi) = p(T^{\pi} = 1 \mid Y(1) = y)$ be the positive ($y=1$) and negative ($y=0$) class predictive performance.\footnote{While our definition of $m^*_y(\pi)$ assumes that $T^{\pi} = 1$ (i.e, the TPR and FPR) the TNR $(y=0)$ and FNR $(y=1)$ can be recovered by taking $1 - R^*(\pi, \pi_0; m_y)$.} For example, a positive false positive rate regret $m^*_{y=0}(\pi) - m^*_{y=0}(\pi_0) > 0$ indicates that the proposed policy recommends tests for healthy patients more frequently than the status quo. Similarly, let $m^*_a(\pi) = \bE[Y(1) = a \mid T^{\pi}=a]$ be the positive ($a=1$) and negative ($a=0$) predictive value of $\pi$.

\textbf{Utility:} Let $u_{a,y} \in \mathbb{R}_{+}$ be the utility of outcome $y$ under action $a$. We let 

$$
m_u^*(\pi)=\bE\left[ \sum_{a, y} u_{ay} \cdot \bI \{ A^{\pi} = a, Y(1) = y \}  \right] 
$$

be the expected utility of $\pi$.\footnote{We model utilities of binary classification outcomes rather than the vector-valued potential outcome $Y(a) \in \{0,1\}^{|A|}$ studied in the standard off policy evaluation setup.} A positive utility regret $m_u^*(\pi) - m_u^*(\pi_0) > 0$ indicates that the updated policy has an overall welfare benefit in comparison to the status quo \citep{rambachan2021identifying}. For example, this measure reflects settings in which turning away sick patients incurs lower utility than testing healthy ones ($u_{01} << u_{10}$) as well as accuracy ($u_{11} = u_{00} = 1$, $u_{10} = u_{01} = 0$).

Because of the selective labels problem, the policy regret is \textit{partially identified} within the interval $R^*(\pi, \pi_0; m) \in [\underline{R}(\pi, \pi_0; m), \overline{R}(\pi, \pi_0; m)]$. Our goal in this work is to recover the most informative regret interval possible given observational data $O = \{(X_i, T^{\pi}_i, D^{\pi_0}_i, Y_i): i=1,...,n \} \sim p(X, D^{\pi_0}, T^{\pi}, Y)$ generated under the status-quo policy.

\subsection{Partial Identification of Policy Performance}

To tightly characterize policy performance differences, we introduce notation that enables us to isolate confounding-related uncertainty in performance measures. We decompose performance measures into a set of sufficient $v$-statistics $\mathbf{v} = \{ v_y(t,d): \forall y, t, d \}$, where each term $v_y(t,d) = p(T^{\pi} =t, D^{\pi_0}=d, Y(1) = y)$ is the joint probability of policy actions and the potential outcome. 

The subset $\textbf{v}_1 = \{v_y(t,1) : \forall y, t\}$ contains all $v$-statistics known from observational data. By consistency (\ref{assumption:consistency}), each term in $\mathbf{v}_1$ is given by 
$$
 p(T=t, D=1, Y(1)=1) = p(T=t, D=1, Y=1).
$$

Similarly, $\textbf{v}_0 = \{v_y(t,0) : \forall y, t\}$ contains all partially identified $v$-statistics. Each term in $\textbf{v}_0$ is bounded in the interval $v_y(t,0) \in [0, p(T=t, D=0)]$ because the potential outcome is unobserved when $D=0$. More generally, we let $v_y(t,0) \in [\underline{v}_y(t,0), \overline{v}_y(t,0)]$ be bounds on partially identified $v$-statistics. We let $\textbf{v}^*_0$ be the true value of partially identified terms under the target distribution and let 
$$
\mathcal{V}(p) = \left\{ \mathbf{v}_0 : \sum_y v_y(t,0) = p(T=t, D=0), \; \forall t  \right\} 
$$
be an uncertainty set of feasible values consistent with the observed data distribution. We can bound policy performance measures over this uncertainty set via 
\begin{align*}
    \underline{m}(\pi_0; \mathcal{V}) &= \min_{\mathbf{v}_0 \in \mathcal{V}(p)} m(\mathbf{v}_0, \mathbf{v}_1; \pi_0), \\
    \overline{m}(\pi_0; \mathcal{V}) &= \max_{\mathbf{v}_0 \in \mathcal{V}(p)} m(\mathbf{v}_0, \mathbf{v}_1; \pi_0),
\end{align*}
where $m(\mathbf{v}^*_0, \mathbf{v}_1; \pi_0)$ is a $v$-statistic decomposition of $m^*(\pi_0)$. In the following section, we show that this reduction to a set of common sufficient statistics enables uncertainty cancellation when comparing across policies.

\section{Regret Bound Identification}\label{sec:bound_identification}

We now introduce our novel approach for comparing policies under confounding. Our approach yields informative regret intervals by eliminating redundant uncertainty that does not contribute to differential policy performance. We first focus on asymptotic performance bounds, and provide finite-sample analyses in Section \ref{sec:bound_estimation}. 

A baseline approach for bounding regret involves partially identifying the performance of each policy individually, and then taking a difference across policy-specific bounds. We refer to this as the baseline regret interval.

\begin{definition}[Baseline regret interval]
The baseline regret interval over $\mathcal{V}(p)$ has lower and upper endpoints: 
\begin{align*}
\underline{R}(\pi, \pi_0; m, \mathcal{V}) &= \underline{m}(\pi; \mathcal{V}) - \overline{m}(\pi_0; \mathcal{V}), \\
\overline{R}(\pi,\pi_0; m, \mathcal{V}) &= \overline{m}(\pi; \mathcal{V}) - \underline{m}(\pi_0; \mathcal{V}).
\end{align*}
\end{definition}

In contrast, our proposed approach directly bounds the difference in the oracle regret $\delta_m(\mathbf{v}^*_0, \mathbf{v}_1) = m(\mathbf{v}^*_0, \mathbf{v}_1;  \pi) - m(\mathbf{v}^*_0, \mathbf{v}_1; \pi_0)$.

\begin{definition}[$\delta$-regret interval]

The $\delta$-regret interval over $\mathcal{V}(p)$ has lower and upper endpoints:
\begin{align*}
\underline{R}_{\delta}(\pi, \pi_0; m, \mathcal{V}) = \min_{\mathbf{v}_0 \in \mathcal{V}(p)} \delta_m(\mathbf{v}_0, \mathbf{v}_1), \\
\overline{R}_{\delta}(\pi, \pi_0; m, \mathcal{V}) = \max_{\mathbf{v}_0 \in \mathcal{V}(p)} \delta_m(\mathbf{v}_0, \mathbf{v}_1).
\end{align*}
\end{definition}

The $\delta$-regret interval yields tighter regret bounds by eliminating redundant uncertainty irrelevant to policy comparison. For example, consider the accuracy regret decomposition
\begin{align*}
\delta_u(\mathbf{v}^*_0, \mathbf{v}_1) &= \sum_y v_y(y,1-y) - v_y(1-y, y).
\end{align*}
Because $v_0(0,0) \in \mathbf{v}_0$ cancels when taking a difference across polices, this term does not contribute uncertainty to the $\delta$-regret interval, as it would with the baseline interval.\footnote{We provide additional intuition for the benefits of our uncertainty cancellation approach in Appendix \ref{appendix:uq_benefits}.} We now provide a theoretical result characterizing the improvement offered by the $\delta$-interval more generally. We let $I_{\delta}(m, \mathcal{V}) = \overline{R}_{\delta}(\pi, \pi_0; m, \mathcal{V}) - \underline{R}_{\delta}(\pi, \pi_0; m, \mathcal{V})$ and $I(m, \mathcal{V}) = \overline{R}(\pi, \pi_0; m, \mathcal{V}) - \underline{R}(\pi, \pi_0; m, \mathcal{V})$ be the length of the $\delta$-regret and baseline regret intervals, respectively.

\begin{theorem}[Regret separation]\label{thm:delta_seperation} 

Let $\Delta(m, \mathcal{V}) = I(m, \mathcal{V}) - I_{\delta}(m, \mathcal{V})$. Then the $\delta$-regret interval offers the following improvement over the baseline regret interval
\begin{align*}
\Delta(m_y, \mathcal{V})  &\geq \frac{2\alpha \cdot v_{y}(1,1)}{(\overline{\gamma}_y)^2}, \;\; \Delta(m_{a=1}, \mathcal{V}) = 0, \\
\Delta(m_u, \mathcal{V}) & \geq 2 \alpha \cdot (u_{00} + u_{01}), \\
\Delta(m_{a=0}, \mathcal{V}) &\geq \frac{2\alpha}{\max\{\psi_{0}(\pi), \psi_{0}(\pi_0)\}}, 
\end{align*}
where $\alpha = \overline{v}_{0}(0, 0) - \underline{v}_{0}(0, 0)$,  $\psi_0(\pi) =  p(A^{\pi} = 0)$, and $\overline{\gamma}_{y} = \sum_{a}\sum_{a'} \overline{v}_{y}(a,a')$.
\end{theorem}

Observe that the improvement offered by the $\delta$-regret interval is proportional to the magnitude of uncertainty in the cancellation term $\alpha = \overline{v}_{0}(0, 0) - \underline{v}_{0}(0, 0)$. Therefore, our approach offers the largest benefit under a high degree of uncertainty around partially identified $v$-statistics.

Because regret measures inherit a monotonic dependence on $v$-statistics, we can compute closed form analytic bounds by maximizing and minimizing regret measures over uncertainty intervals around partially identified terms $[\underline{v}_y(t,0), \overline{v}_y(t,0)]$. We provide $\delta$-regret and baseline regret intervals for policy performance measures in Appendix \ref{appendix:asymptotic_regret_identification}.

\section{Mapping Causal Assumptions to Informative Regret Bounds}\label{sec:pid_framework}

The regret interval we recover without imposing additional causal assumptions on the data generating policy tends to be uninformative. Therefore, we now introduce a \textit{pointwise bounding functions} assumption which we use to tighten regret intervals by shrinking the uncertainty set around partially-identified $v$-statistics. We show that this assumption is implied by a range of causal assumptions studied in prior OPE literature (e.g., MSM, Rosenbaum’s $\Gamma$, IV) in Appendix \ref{appendix:assumption_extensions}. Thus, our $\delta$-regret interval yields tighter uncertainty quantification than the baseline interval for any causal assumption implying pointwise bounding functions (Figure \ref{fig:assumption_flow}).

\begin{assumption}[Pointwise bounding functions]\label{assumption:pbf}

Let $\underline{\tau}, \overline{\tau}: \mathcal{X} \rightarrow \left[ 0, 1 \right]$ be a pair of bounding functions satisfying 
$$
\underline{\tau}(x) \leq \bE[ Y(1) \mid D=0, X=x ] \leq \overline{\tau}(x), \;\; \forall x \in X. 
 $$   
\end{assumption}
Intuitively, this assumption requires that each individual screened out under the status quo policy has risk which is bounded by [$\underline{\tau}(x), \overline{\tau}(x)$]. \citet{rambachan2022counterfactual} show that a variant of this assumption is implied by the IV framework, the MSM, and Rosenbaum's $\Gamma$ sensitivity model.\footnote{Our pointwise bounding function assumption is also related to the notion of ``bounded outcomes'' studied by \citet{manski1990nonparametric} and ``natural outcomes'' studied by \citet{pearl2009causality}.} Next, we show that bounding functions can be used to construct intervals on partially identified $v$-statistics.

\begin{figure}
    \centering
    \includegraphics[width=\linewidth]{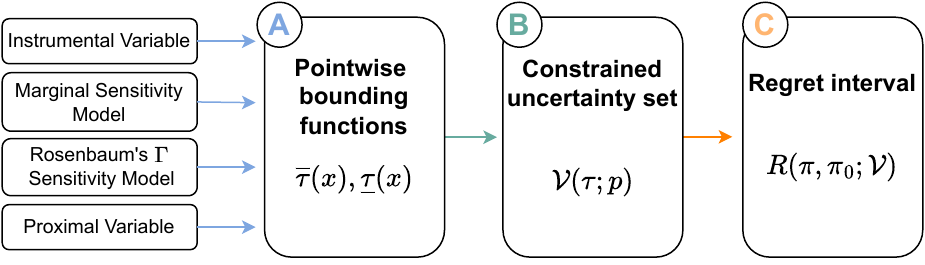}
    \caption{Flow of assumptions in our framework. (A) Traditional causal assumptions imply \textit{pointwise bounding functions} on the unobserved outcome (Appendix \ref{appendix:assumption_extensions}); (B) Pointwise bounding functions imply constrained uncertainty sets (Lemma \ref{lemma:assumption_mapping}); (C) Constrained uncertainty sets imply policy regret bounds (Section \ref{sec:bound_identification}).}
    \label{fig:assumption_flow}
\end{figure}

\begin{lemma}[Assumption mapping]\label{lemma:assumption_mapping}

Let $\underline{\tau}, \overline{\tau}: \mathcal{X} \rightarrow \left[ 0, 1 \right]$ be a pair of bounding functions satisfying Assumption \ref{assumption:pbf}. Then partially identified $v$-statistics are bounded by
$$
\bE[\underline{\tau}(x) \cdot e_0(x) \cdot \pi_t(x)] \leq v_1(t,0) \leq \bE[\overline{\tau}(x) \cdot e_0(x) \cdot \pi_t(x)],
$$
where $e_d(x) = p(D=d \mid X=x)$ and $\pi_t(x) = p(T=t \mid X=x)$. 
\end{lemma}

Lemma \ref{lemma:assumption_mapping} provides a convenient approach for bounding $\mathbf{v}^*_0$ because, given a pair of bounding functions, it only requires knowledge of \textit{nominal} (i.e., confounded) status quo policy probabilities $e_d(x)$ and the new policy $\pi_t(x)$. Nominal probabilities and bounding functions can be learned from observational data, while $\pi_t(x)$ is known in advance by model developers. 
Our goal is to tighten regret intervals by using bounding functions to shrink the uncertainty set around $\mathbf{v}^*_0$. Let $\mathcal{H}(v_1(t,0); \tau) = [\underline{v}_y(t,0), \overline{v}_1(t,0)]$ be the interval implied by Lemma \ref{lemma:assumption_mapping} and let $\rho_{td}=p(T=t, D=d)$. We define the constrained uncertainty set over $\mathbf{v}^*_0$ as 
\[
\mathcal{V}(p;\tau) = \left\{v_y(t,0) \colon
\begin{aligned}
& \quad v_1(0,0) \in \mathcal{H}(v_1(0,0);\tau) & \\
& \quad v_0(0,0) = \rho_{00} - v_1(0,0) & \\
& \quad v_1(1,0) \in \mathcal{H}(v_1(1,0);\tau) & \\
& \quad v_0(1,0) = \rho_{10} - v_1(1,0) &
\end{aligned}
\right\}.
\]

We only use intervals $\mathcal{H}(v_1(t,0);\tau)$ for this set definition because an analogous bound is implied over $v_0(t,0)$ by the constraint requiring that $\sum_y v_y(t,0) = \rho_{t0}$. We quantify the size of $\mathcal{V}(p;\tau)$ via the Lebesgue measure, which is defined as the cartesion product over partially identified intervals $\lambda(\mathcal{V}(p;\tau)) = |\overline{v}_1(1,0)-\underline{v}_1(1,0)||\overline{v}_1(0,0)-\underline{v}_1(0,0)|$. We now show that $\mathcal{V}(p;\tau)$ is the smallest uncertainty set one could construct based on the observational distribution and a pair of bounding functions.

\begin{theorem}[Minimality]\label{thm:minimality} $\mathcal{V}(p; \tau)$ is the minimal uncertainty set---that is, no strict subset of $\mathcal{V}(p; \tau)$ is consistent with $p(X, D, T, Y)$ and user-specified bounding functions.
\end{theorem}

\begin{algorithm}[tb]
   \caption{Plug-in regret bound estimator}
   \label{algorithm:plug-in}
\begin{algorithmic}
   \STATE {\bfseries Input:} Data $\mathcal{O} = \{(X_i, D^{\pi_0}_i, T^{\pi}_i, Y_i)\}_{i=1}^{n} \sim p$, Folds $K$
   \FOR{$k=1$ {\bfseries to} $K$}

   \STATE Learn $\hat{\eta}_{k} = (\hat{e}_{1,k}, \hat{\tau}_{k})$ using $\mathcal{O}_{-k}$

   \STATE Estimate $\mathcal{H}(\hat{v}_{1,k}(t,0); \hat{\tau})$ via e.q. (\ref{eq:plugin}) using $\mathcal{O}_{k}$

   \STATE Construct $\hat{\mathcal{V}}_k(p;\hat{\tau})$ 
   
   \STATE Compute $\underline{\hat{R}}_{\delta,k}(\pi, \pi_0; m, \hat{\mathcal{V}}_k), \hat{\overline{R}}_{\delta,k}(\pi, \pi_0; m, \hat{\mathcal{V}}_k)$
   
   \ENDFOR
   \STATE  Set $\underline{\hat{R}}_{\delta}(\pi, \pi_0; m, \hat{\mathcal{V}}) = \frac{1}{K} \sum_{k} \underline{\hat{R}}_{\delta,k}(\pi, \pi_0; m, \hat{\mathcal{V}}_k)$, \;
   \STATE Set $\hat{\overline{R}}_{\delta}(\pi, \pi_0; m, \hat{\mathcal{V}}) = \frac{1}{K} \sum_{k} \hat{\overline{R}}_{\delta,k}(\pi, \pi_0; m, \hat{\mathcal{V}}_k)$
\end{algorithmic}
\end{algorithm}

\section{Regret Bound Estimation}\label{sec:bound_estimation}

In this section, we construct finite sample estimates of regret intervals using observational data $\mathcal{O} = \{(X_i, D^{\pi_0}_i, T^{\pi}_i, Y_i)\}_{i=1}^{n} \sim p$ collected under the status quo policy. We estimate bounds by applying closed-form expressions for the $\delta$-regret interval over an estimate of the uncertainty set $\mathcal{\hat{V}}(p; \hat{\tau})$. We can directly compute finite-sample estimates $\hat{\mathbf{v}}_1$ for identified $v$-statistics by taking a sample average. As a result, the technical challenge underpinning regret bound estimation boils down to estimating the intervals $\mathcal{H}(\hat{v}_y(t,0); \hat{\tau}) = [\underline{\hat{v}}_y(t,0), \hat{\overline{v}}_y(t,0)]$ around $\mathbf{v}^*_0$. 

\subsection{Plug-in Estimator}

A direct approach for bound estimation involves learning $\hat{\eta} = (\hat{e}, \hat{\tau})$, then estimating $\mathcal{H}(\hat{v}_1(t,0); \hat{\tau})$ over a held-out sample by invoking Lemma \ref{lemma:assumption_mapping}.\footnote{In line with prior literature on CATE estimation \citep{kennedy2023towards}, we call this a plug-in approach because it involves directly substituting in empirical estimates for the nuisance functions.}\footnote{The term ``nuisance function'' refers to a learned function of the data which differs from our target quantity of interest.} We assume that $\pi(x)$ is known by model developers and thus does not need to be learned. We split the data into $K$ disjoint folds, where we denote $\mathcal{O}_k$, $\mathcal{O}_{-k}$, as the sample inside and outside of fold $k$, respectively. We then define the plug-in estimator over $\mathcal{O}_k$ as

\begin{equation}\label{eq:plugin}
   \hat{\overline{v}} = \frac{1}{|\mathcal{O}_k|}\sum_{x_i \in \mathcal{O}_k} \pi(x_i) \cdot \hat{e}(x_i) \cdot \hat{\overline{\tau}}(x_i) 
\end{equation}

where $\hat{\eta}$ is learned over fold $\mathcal{O}_{-k}$.\footnote{We suppress dependence on $y, t, d$ and present analysis of the upper bound to ease notation. An analogous argument holds for $\hat{\underline{v}}$.} We recover regret estimates at full data efficiency via a cross-fitting approach outlined in Algorithm \ref{algorithm:plug-in}. The following result shows consistency of $\hat{\overline{v}}$ and characterizes its convergence rate to the true upper bound $\overline{v}$, where we define ${\norm{f}^2 := \int (f(x))^2 dP(x)}$ to be the squared $L_2(O)$ norm of a function $f$.

\begin{theorem}\label{thm:pl_convergence} 
Let $f(x)=e(x) \cdot \overline{\tau}(x)$ and assume $|| \hat{f} - f  || = o_p(1)$. Then the plug-in estimator satisfies
$$
\hat{\overline{v}} - \overline{v} = O_{P}\left(\| e(x) - \hat{e}(x) \| + \| \overline{\tau}(x) - \hat{\overline{\tau}}(x) \|\right) + O_{P}\left(\frac{1}{\sqrt{n}}\right).
$$
\end{theorem}

Because policy performance measures are linear combinations ($m^*_u$) or ratios ($m^*_a$, $m^*_y$) of $v$-statistics, this result also characterizes the convergence rate of the estimated upper and lower regret bounds. Observe that the convergence rate of the plug-in estimator inherits the rate of the nuisance functions. As a result, $\hat{v}$ will tend to converge slowly when using flexible machine learning methods to fit nuisance functions under no parametric assumptions. 

\begin{figure*}
    \centering
    \includegraphics[width=.9\linewidth]{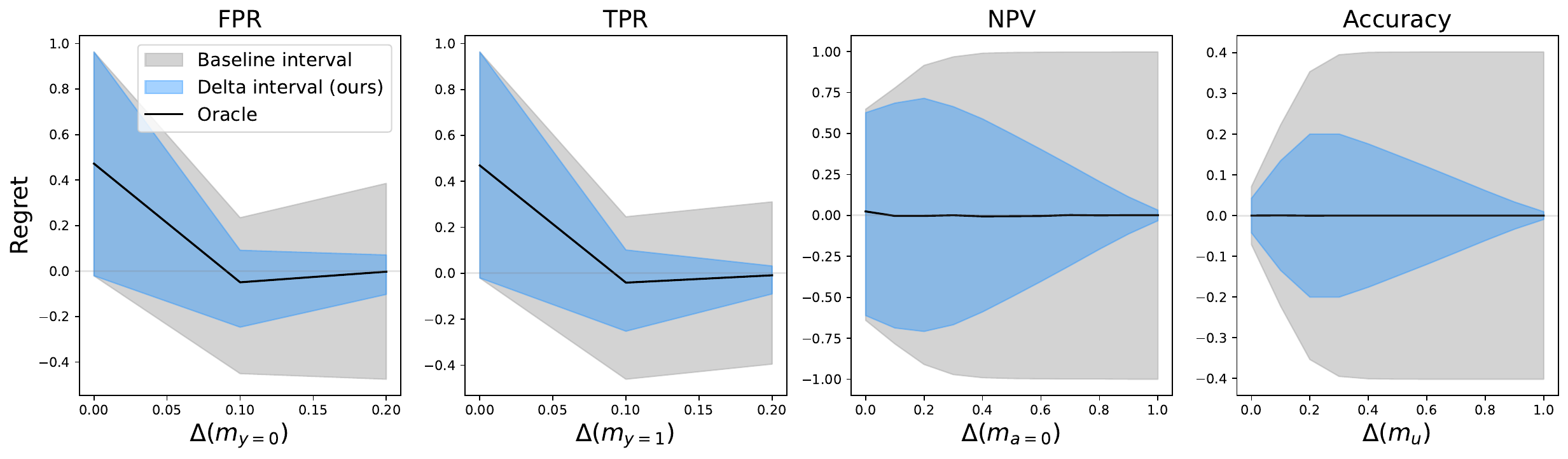}
    \caption{Improvement in bounds offered by the $\delta$-regret interval over the baseline interval. We systematically vary the relative size of $v$-statistics and plot bounds as a function of interval improvement $\Delta(m) = I(m) - I_{\delta}(m)$ characterized by Theorem \ref{thm:delta_seperation}.}
    \label{fig:bound-improvement}
\end{figure*}

\begin{figure*}[ht]
    \centering
    \begin{minipage}{0.48\textwidth}
        \centering
        \includegraphics[width=\linewidth, trim=50 10 100 60, clip]{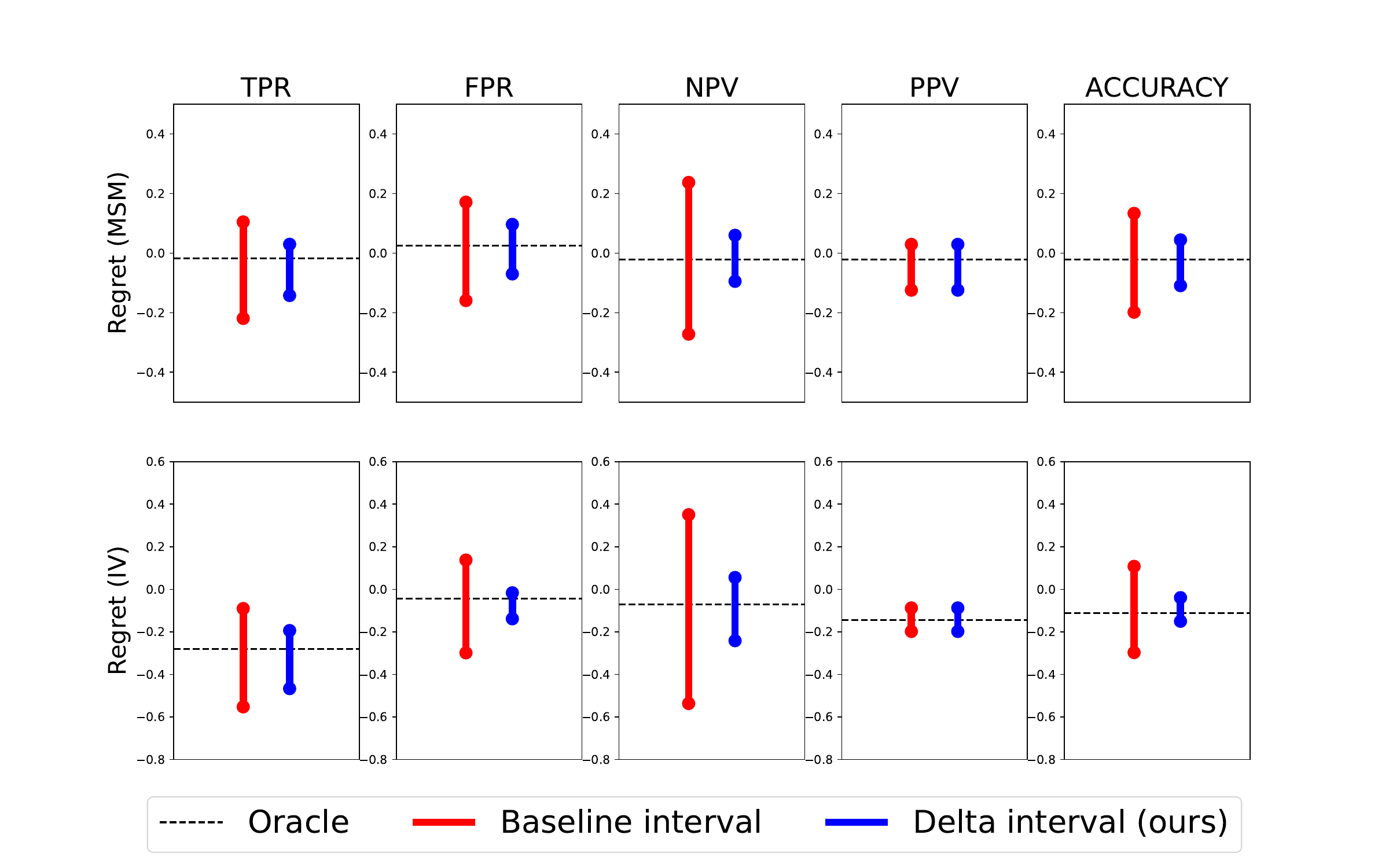}
        \caption{Comparison between $\delta$-regret and baseline regret interval end-points averaged over $N=20$ trials and $N_s=20,000$ samples. First row leverages an MSM identification assumption, while the second leverages an IV assumption.}
        \label{fig:delta_standard_comparison}
        \vspace{-4mm}
    \end{minipage}\hfill
    \begin{minipage}{0.48\textwidth}
        \centering
        \includegraphics[width=.9\linewidth]{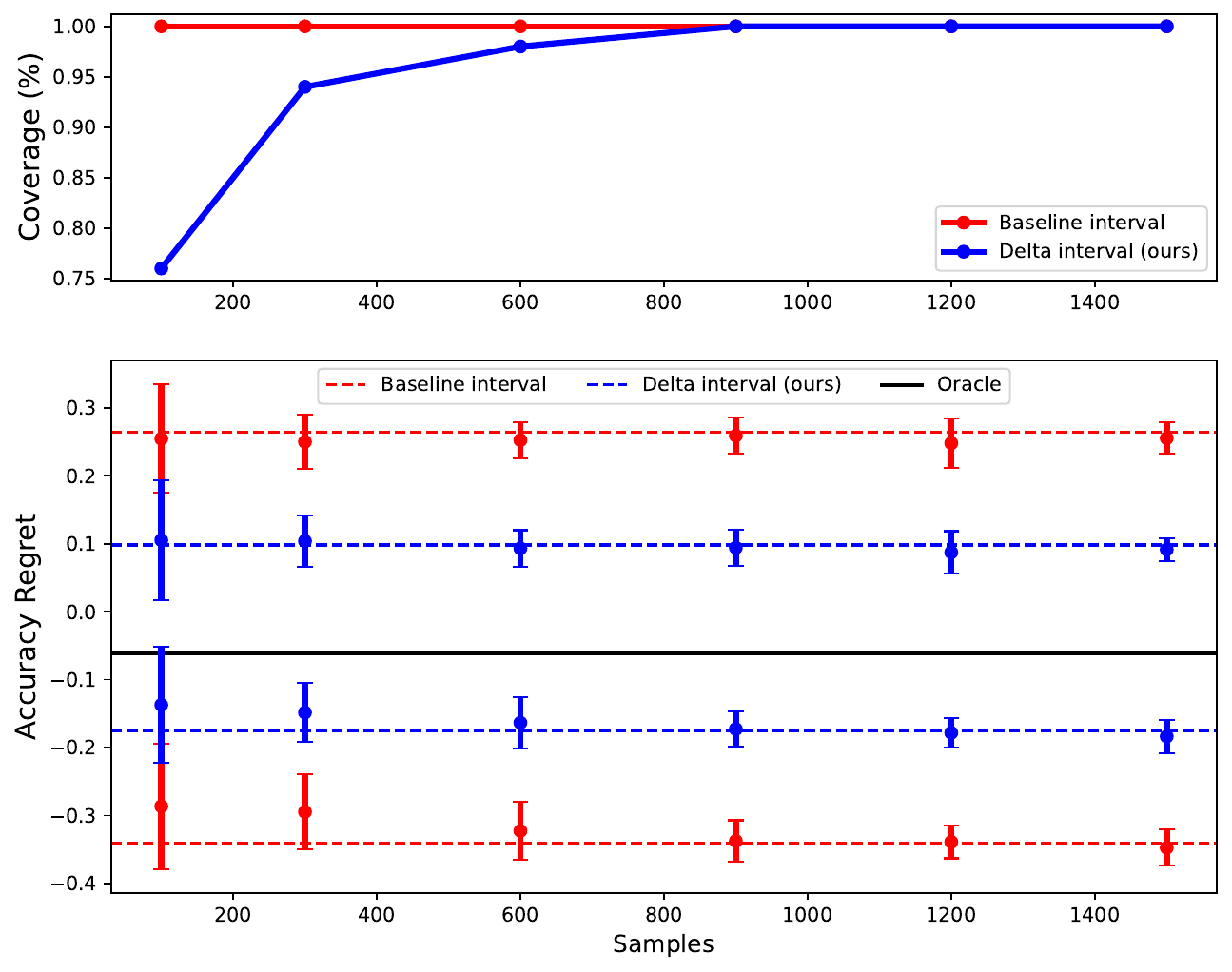}
        \caption{Top: Coverage of accuracy regret interval estimates as a function of total sample size. Bottom: 95\% bootstrap confidence intervals around upper and lower regret bounds over $N=25$ trials. Solid line indicates the oracle regret.}
        \label{fig:pl_estimation}
    \end{minipage}
\end{figure*}

\subsection{Doubly Robust Estimator}

Therefore, we propose a doubly robust approach for estimating $\hat{v}$ that corrects the bias in the plug-in estimator. Because doubly robust estimators  have an error term which is a product of nuisance function errors, they attain fast convergence rates, even when estimating nuisance functions at slow non-parametric rates \citep{kennedy2021semiparametric}. For example, to attain $n^{-1/2}$ rates for the estimator, it is sufficient to estimate nuisance functions at $n^{-1/4}$, a rate achieved by many non-parametric machine learning techniques. The form of the doubly robust estimator depends on the pointwise bounding functions implied by a causal assumption.  

We provide the doubly robust estimator for the commonly studied MSM. The MSM assumes there is some $\Lambda \geq 1$ such that the odds ratio $\pi_0(X,U))/(1-\pi_0(X,U)) \cdot (1-e(X))/e(X)$ lies within $[\Lambda^{-1}, \Lambda]$ (See \ref{assumption:msm}).
%
We define the doubly robust estimator over $\mathcal{O}_k$ as 
\begin{equation}\label{eq:doublyrobust}
   \hat{\overline{v}}_{DR} = \frac{1}{|\mathcal{O}_k|} \sum_{x_i \in \mathcal{O}_k}  \phi(O_i; \hat e,  \pi, \hat \mu_1)
\end{equation}
where  $\phi(O ; e, \pi, \mu_1) =  D \cdot Y\pi(X) \cdot\Lambda + (T-\pi(X))\cdot e(X) \cdot \mu_1(X) \cdot \Lambda$ and $\mu_1(x) = \bE[Y(1) \mid D=1, X=x]$.

Algorithm \ref{algorithm:plug-in} can be leveraged for doubly-robust estimation of policy regret bounds by substituting e.q. \ref{eq:doublyrobust} for estimation of $\mathcal{H}(\hat{v}_{1,k}(t,0); \hat{\tau})$. We show in Appendix \ref{appendix:estimation_results} that this estimator has second order error in nuisance estimation error.

\begin{figure*}[ht]
  \centering
  \begin{minipage}[t]{0.45\linewidth}
    \includegraphics[width=1\linewidth]{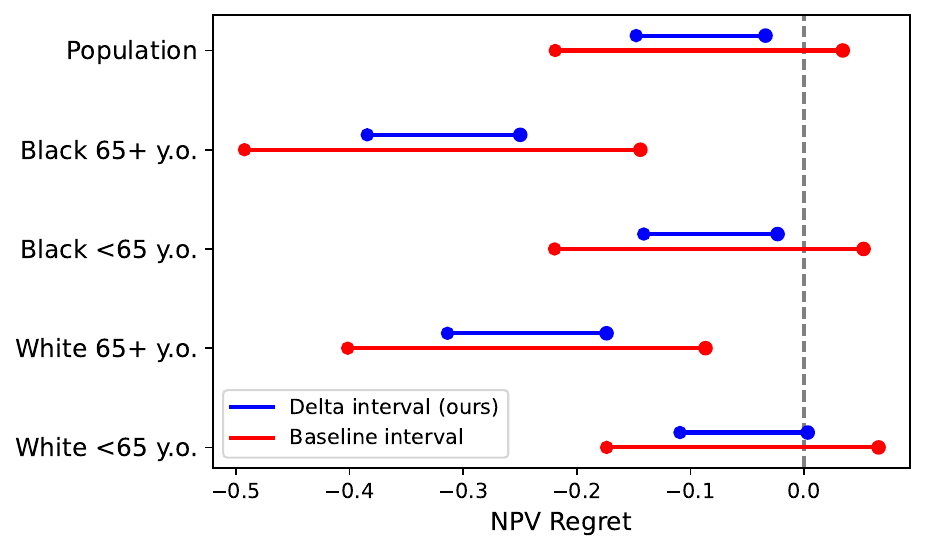}
    \caption{Comparison of the $\delta$-regret interval against the baseline for the NPV ($m_{a=0}$). The top row indicates full population bounds, while lower rows subpopulation bounds.}
    \label{fig:subgroup-plot}
  \end{minipage}
  \hspace{0.5cm} 
  \begin{minipage}[t]{0.45\linewidth}
    \includegraphics[width=1\linewidth]{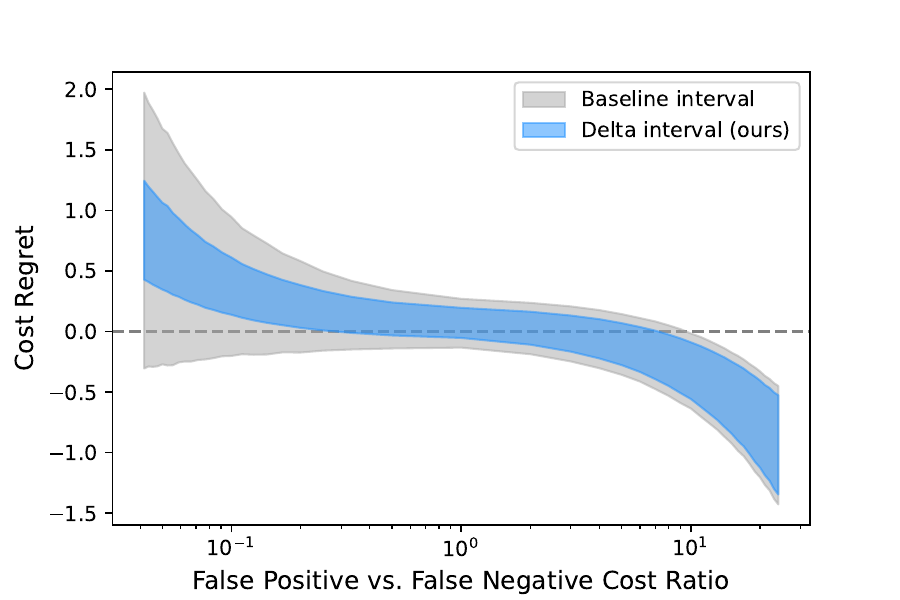}
    \caption{Expected cost regret as a function of false positive to false negative cost ratio. Left hand side ($10^{-1}$) indicates false negative cost ten times greater than false positive cost.}
    \label{fig:cost-regret}
  \end{minipage}
\end{figure*}

\section{Numerical Experiments}\label{sec:numeric_experiments}

We now show numerical experiments comparing decision policies under a known data-generating mechanism.

\textbf{Regret interval characterization.} We first characterize the uncertainty reduction offered by the $\delta$-regret interval. We simulate a range of $v$-statistic decompositions by randomly sampling $\mathbf{v}$ satisfying $\sum_{y,t,d} v_y(t,d) = 1$ and $v_0(t,d) + v_1(t,d) = \rho_{td}, \forall t, d$. This yields a collection of decompositions corresponding to valid observational distributions. We compute the analytic improvement in bounds defined in Theorem \ref{thm:delta_seperation} for each $\mathbf{v}$ and plot the results in Figure \ref{fig:bound-improvement}. We observe a monotone improvement in bounds as a function of $\Delta(m, \mathcal{V})$, with no improvement when $\Delta(m, \mathcal{V}) = 0$ and significant improvement when $\Delta(m, \mathcal{V}) = 0$ is large.

\textbf{Synthetic data experiment.} We next validate our framework by simulating data consistent with two common assumptions for off-policy evaluation: the MSM (\ref{assumption:msm}) and IV (\ref{assumption:instrument}). We draw $v$ covariates $X_i \sim \mathcal{N}(0, I_v)$, and $u$ confounders $U_i \sim \mathcal{N}(0, I_u)$, and let $V_i = (X_i, U_i)$.  We parameterize policy, outcome, and instrument probability functions with coefficients $W_{\pi_0} \in \mathbb{R}^{v+u}$, $W_{\pi} \in \mathbb{R}^{v}$, $W_{\mu_1} \in \mathbb{R}^{v+u}$, $W_{z} \in \mathbb{R}^{v \times z}$, respectively, with each drawn from a uniform distribution. We sample data from the probability functions 
\begin{align*}
\gamma(X_i) &:= \sigma'(X_i \times W_{z} ), \pi(X_i) := \sigma(X_i \times W_{\pi})\\
\pi_0(V_i, Z_i) &:= \sigma(V_i \times W_{\pi_0} + \beta_0 \cdot Z_i) \\
\mu_1(V_i, Z_i) &:= \sigma(V_i \times W_{\mu_1} + \beta_1 \cdot Z_i) \label{eq:dgp_iv_exclusion} \\
\mu_0(V_i, Z_i) &:= \begin{cases}
   \Lambda^* \cdot \mu_1(V_i), \Lambda^* \in U(\Lambda^{-1}, \Lambda) & \text { (MSM) } \\
   \sigma(V_i \times W_{\mu_0} + \beta_1 \cdot Z_i) & \; \text{ (IV) } \\
\end{cases} \\
\mu(V_i, Z_i) &:= \mu_1(V_i, Z_i)  \pi_1(V_i, Z_i) + \mu_0(V_i, Z_i)  \pi_0(V_i, Z_i)
\end{align*}
where $\sigma(x) = \frac{1}{1+e^{-x}}$ and $\sigma'(x)$ is the softmax function. We use Algorithm \ref{algorithm:plug-in} with the plug-in estimator to learn estimates estimates of the $\delta$-regret interval. We compare against the baseline regret interval by applying bounds provided in Appendix \ref{appendix:asymptotic_regret_identification} over the same $\hat{\mathcal{V}}(p, \hat{\tau})$ used to estimate the $\delta$-regret. See Appendix \ref{appendix:experiments} for additional setup details.

Figure \ref{fig:delta_standard_comparison} provides a comparison of the $\delta$-regret interval and baseline regret interval across five policy performance measures. In line with our theoretical results, we obtain tighter bounds from the $\delta$-regret across all policy performance measures apart from the positive predictive value (PPV). We show the oracle regret evaluated via $Y(1)$ in black. In Figure \ref{fig:pl_estimation}, we plot coverage of the estimated $\delta$-regret interval evaluated against the true population value, denoted by dashed lines. The bottom panel shows $95\%$ bootstrap confidence intervals around the true upper and lower $\delta$-regret and baseline regret intervals. These results show that estimates concentrate around the true regret interval as the number of samples increases. We observe that the baseline interval yields better coverage of the oracle regret in small sample settings because it is more conservative and thus more tolerant to bias around estimates of the asymptotic interval end points. Appendix \ref{appendix:experiments} contains additional experimental results stress-testing coverage under causal assumption violations.

\section{Real-World Application: Comparing Healthcare Enrollment Policies} \label{sec:realworld_experiments}

We now illustrate how our framework can be used to compare alternative healthcare enrollment policies. In medical settings, providers routinely screen patients for diseases and enroll high-risk individuals in preventative care programs. However, it is challenging to assess the performance improvement of a proposed policy because outcomes are only observed among patients enrolled under the existing policy \citep{daysal2022economic}. Confounding is a challenge in this setting because physicians often make decisions using unobserved information \citep{mullainathan2022diagnosing}. 

We leverage data released by \citet{obermeyer2019dissecting} to construct an enrollment policy comparison task. This dataset contains $\approx 48,000$ records, where each entry consists of a patient evaluated for enrollment in a high-risk care management program.\footnote{We use a synthetic version of the original dataset, which was released by \citet{obermeyer2019dissecting} to protect patient confidentiality. This dataset preserves the means and covariances of the original data.} We let $\pi_0$ be the historical enrollment policy consisting of physician decisions informed by an algorithmic risk score and interactions with the patient. We take $\pi$ to be the algorithmic policy which makes decisions by thresholding predictions of patient cost from the clinical decision support tool. The goal in this task is to assess whether the algorithm-only policy would improve upon the status quo human+algorithm policy used to collect data. We detail our setup in Appendix \ref{appendix:realworld_setup}.

\textbf{Results.} Because we do not have access to physician identifiers which can be used as an instrument, we leverage the MSM assumption for identification. We compare the baseline and $\delta$-regret intervals for the NPV in Figure \ref{fig:subgroup-plot}. The top row shows the population regret, while the bottom rows show a breakdown across subgroups.\footnote{Intervals include $95\%$ confidence intervals estimated around regret bound endpoints. We omit these from the plot for readability.} The $\delta$-regret interval excludes zero across all subpopulations apart from White patients under 65. As a result, an analysis conducted via the $\delta$-regret interval supports an interpretation that the algorithmic policy reduces the NPV in comparison to the status quo human+algorithm policy.\footnote{Given the narrow scope of our analyses and limitations of synthetic data, our findings are not intended to be a conclusive assessment of the policies evaluated by \citet{obermeyer2019dissecting}.} \textbf{Because the baseline regret interval contains zero among all patient populations except those aged 65 and over, this bounding approach supports weaker claims regarding the relative performance of decision policies.} 

In Figure \ref{fig:cost-regret}, we plot the $\delta$-regret and baseline intervals as a function of the ratio between false positive and false negative costs. The left hand region of this figure ($\approx 10^{-1}$) reflects a setting in which false negatives are ten times more costly than false positives. This regime is realistic in our healthcare enrollment scenario when turning a sick patient away from the program incurs more harm than enrolling a healthy patient. We observe the greatest improvement from our approach in the high cost false negative setting because the $v_0(0,0)$ term which cancels under our approach is heavily weighted in the expected cost calculation. Because the $\delta$-regret interval excludes zero in this regime, this supports an interpretation that the proposed policy has a higher expected cost than the status quo.

\section{Conclusion}

In this work, we propose the first framework supporting predictive performance comparisons of decision-making policies under confounding. Our approach is intended to support pre-deployment evaluations of proposed policies under a flexible set of causal assumptions. Our approach addresses sources of confounding-related uncertainty which impact model evaluations and, where possible, reduces this uncertainty via technically novel partially identification approaches. Our uncertainty cancellation approach may prove useful for more tightly characterizing performance differences under other uncertainty sources, such as missing protected attributes \citep{kallus2022assessing} and measurement error \citep{fogliato2020fairness}.

\section*{Impact Statement}

Our framework is designed to support pre-deployment evaluations of proposed decision policies. While our framework is intended to faithfully represent confounding-related sources of uncertainty impacting decision policy evaluations, it does not speak to broader measurement challenges \citep{guerdan2023ground} and ethical questions underpinning the introduction of an algorithmic system \citep{coston2023validity, rittel1973dilemmas}. Given the high-stakes contexts in which some policies are deployed (e.g., lending, healthcare, education), our framework should be applied carefully as part of a multifaceted impact assessment.

\section*{Acknowledgements}

\lgedit{We thank our anonymous reviewers, the attendees of the NeurIPS
2023 Workshop on Regulatable Machine Learning, and the attendees of Carnegie Mellon University's Fairness, Explainability, Accountability, and Transparency (FEAT) reading group for their helpful feedback. This work was supported by an award from the UL
Research Institutes through the Center for Advancing Safety of
Machine Intelligence (CASMI) at Northwestern University and the National Science Foundation Graduate Research Fellowship Program (Award No. DGE1745016).}

\bibliography{refs}
\bibliographystyle{icml2022}

\newpage
\appendix
\onecolumn
\newpage
\section{Asymptotic Regret Bounds}\label{appendix:asymptotic_regret_identification}

In this appendix, we derive $\delta$-regret and baseline regret intervals for policy performance measures. We prove Theorem \ref{thm:delta_seperation} in $\S$ \ref{subsec:seperation_proof}.

\subsubsection{$\delta$-regret bounds on utility regret.}

\begin{lemma}[$\delta_u$-regret bounds]\label{thm:udelta}

Let $m_u^*(\pi)$ be the expected utility of $\pi$ given utility values $u_{ay} \geq 0$. Let $a' = 1-a$, $\lambda_{ay} = u_{ay} - u_{a'y}$, and $\tilde{y} = \bI\{ \lambda_{11} > \lambda_{10} \}$. Then 
$$
\delta_u(\mathbf{v}_0, \mathbf{v}_1) = \sum_{ay} \lambda_{ay} \cdot v_y(a, a')
$$
and for all uncertainty sets $\mathcal{V}(p;\tau)$, the upper  $\delta_u$-regret bound is given by
$$
\overline{R}_{\delta}(\pi, \pi_0; m_u, \mathcal{V}) = \delta_u(\overline{\mathbf{v}}_0, \mathbf{v}_1), \text{ where } \overline{\mathbf{v}}_0 \in \argmax_{v_{\tilde{y}}(1,0)} \mathcal{V}(p; \tau). 
$$

\end{lemma}

\begin{proof}
\begin{align*}
m^*_u(\pi) &= \bE\left[ \sum_{t, y} u_{ty} \cdot \bI \{ T^{\pi} = t, Y(1) = y \} \right] \\
           &= \sum_{t,y} u_{ty} \cdot \bE[\bI \{T=t, Y(1)=y \}] \\
           &= \sum_{t,y} u_{ty} \cdot p(T=t, Y(1)=y) \\
           &= \sum_{t,y} u_{ty} \cdot (v_y(t,0) + v_y(t,1)). \\
\end{align*}
By the same argument, $m^*_u(\pi_0) = \sum_{d,y} u_{dy} \cdot (v_y(0,d) + v_y(1,d))$. Let $\lambda_{ay} = u_{ay} - u_{a'y}$ and $\tilde{y} = \bI\{ \lambda_{11} > \lambda_{10} \}$. Therefore
\begin{align*}
\delta_u(\mathbf{v}^*_0, \mathbf{v}_1) &= m^*_u(\pi) - m^*_u(\pi_0) \\
&= \sum_{t,y} u_{ty} \cdot (v_y(t,0) + v_y(t,1)) - \sum_{d,y} u_{dy} \cdot (v_y(0,d) + v_y(1,d)) \\
&= \sum_{a,y} (u_{ay} - u_{a'y}) \cdot v_y(a, a') \\
&= \sum_{a,y} \lambda_{ay} \cdot v_y(a, a')
\end{align*}

Recall that $\rho_{td} = v_1(t,0) + v_0(t,0)$ by the constraint on $p(\cdot)$. Therefore, the regret is bounded within the interval 
\begin{align*}
R^*(\pi, \pi_0; m_y) &\leq \begin{cases}
\sum_{y} \lambda_{0y} \cdot v_y(0,1) + \lambda_{11} \cdot \overline{v}_1(1,0) + \lambda_{10} \cdot \underline{v}_{0}(1,0) & \lambda_{11} > \lambda_{10} \\
\sum_{y} \lambda_{0y} \cdot v_y(0,1) + \lambda_{11} \cdot \underline{v}_1(1,0) + \lambda_{10} \cdot \overline{v}_{0}(1,0) & \lambda_{11} \leq \lambda_{10} \\
\end{cases}\\
 &= \sum_y \lambda_{0y} \cdot v_y(0,1) + \lambda_{1\tilde{y}} \cdot \overline{v}_{\tilde{y}}(1,0) + \lambda_{1,1-\tilde{y}} \cdot \overline{v}_{1-\tilde{y}}(1,0), \text{ where } \tilde{y} = \bI\{ \lambda_{11} > \lambda_{10} \}
\end{align*}

The lower regret bound is symmetric.

\end{proof}

\newpage
\subsubsection{$\delta$-regret bounds on positive and negative class performance}

\begin{lemma}[$\delta_y$ regret bounds]\label{thm:ydelta}

Suppose that $p(Y(1) = y) > 0$. Then $\overline{R}_{\delta}(\pi, \pi_0; m_y, \mathcal{V}) = \delta_y(\overline{\mathbf{v}}_0, \mathbf{v}_1)$, where 
$$
\delta_y(\mathbf{v}_0, \mathbf{v}_1) = \frac{v_{y}(1,0) -  v_{y}(0, 1)}{ v_{y}(0,0) + v_{y}(1,0) + v_{y}(0,1) + v_{y}(1,1)},
$$
$$
\overline{\mathbf{v}}_0 =
\begin{cases}
   \arg \min\limits_{{v_{y}(0,0)}} \max\limits_{{v_y}(1,0)} \mathcal{V}(p; \tau), \;  \overline{v}_y(1,0) - v_y(0, 1) \geq 0 \\
   \arg \max\limits_{v_{y}(0,0)} \max\limits_{{v_y}(1,0)} \mathcal{V}(p; \tau), \; \underline{v}_y(1,0) - v_y(0,1) < 0\\
\end{cases}
$$
\end{lemma}

Next, we derive $\delta$-regret bounds on $m_y$.

\begin{proof}

\begin{align*}
\delta(\mathbf{v}^*_0, \mathbf{v}_1) &= m^*_y(\pi) - m^*_y(\pi_0) \\
&= p(T^{\pi} = 1 \mid Y(1) = y) - p(D^{\pi_0} = 1 \mid Y(1) = y) \\
&= \frac{p(T=1, Y(1) = y)}{p(Y(1) = y)} - \frac{p(D=1, Y(1) = y)}{p(Y(1) = y)} \\
&= \frac{v_y(1,0) + v_y(1,1)}{p(Y(1) = y)} - \frac{v_y(0,1) + v_y(1,1)}{p(Y(1) = y)}\\
&= \frac{v_y(1,0) - v_y(0,1) }{v_y(0,0) + v_y(1,0) + v_y(0,1) + v_y(1,1)} \\
\end{align*}
Therefore, 
\begin{equation*}
R^*(\pi, \pi_0; m_y) <= \begin{cases}
\dfrac{\overline{v}_y(1,0) - v_y(0,1) }{\underline{v}_y(0,0) + \overline{v}_y(1,0) + v_y(0,1) + v_y(1,1)}, & \overline{v}_y(1,0) - v_y(0,1) > 0 \\
\dfrac{\overline{v}_y(1,0) - v_y(0,1) }{\overline{v}_y(0,0) + \overline{v}_y(1,0) + v_y(0,1) + v_y(1,1)}, & \overline{v}_y(1,0) - v_y(01) <= 0
\end{cases}
\end{equation*}

The lower regret bound is symmetric.

\end{proof}

\begin{remark}
    The result above implies bounds on the FPR ($y=0$) and TPR ($y=1$) regret. Taking 
    $[-\overline{R}_{\delta}(\pi, \pi_0; \mathbf{v}, m_y), -\underline{R}_{\delta}(\pi, \pi_0; \mathbf{v}, m_y)]$ recovers bounds on the TNR ($y=0$) and FNR ($y=1$) regret. 
\end{remark}

\newpage
\subsubsection{$\delta$-regret bounds on positive and negative predictive value}

\begin{lemma}
Let $m_a(\pi) = p (Y(1)= a \mid A^{\pi}=a)$ be the positive $(a=1)$ or negative $(a=0)$ predictive value of $\pi$. Let $\rho_{td} = p(T=t, D=d)$, $\sigma(a) = (1-2  a) (\rho_{10} - \rho_{01})$ and $\psi_a(\pi) = p(A^{\pi} = a)$ and assume that $(D, T) \sim p^*(\cdot)$ satisfies $ p(D=a) \cdot p(T=a) > 0$. Then
\begin{align*}
\delta_a(\mathbf{v}_0, \mathbf{v}_1) = \frac{\sigma(a) \cdot v_a(a, a)  + \psi_a(\pi_0) \cdot v_{a}(a, a') - \psi_a(\pi) \cdot v_a(a', a)}{\psi_a(\pi) \cdot \psi_a(\pi_0)}.
\end{align*}

Additionally, for all constrained uncertainty sets $\mathcal{V}(p;\tau)$, the upper $\delta_a$-regret bound is given by
\begin{align*}
\overline{R}_{\delta}(\pi, \pi_0; m_{a=0}, \mathcal{V}) &= \delta_{a=0}(\overline{\mathbf{v}}_0, \mathbf{v}_1), \; \text{ where } \overline{\mathbf{v}}_0 =
\begin{cases}
   \arg \max\limits_{{v_{y}(0,0)}} \min\limits_{{v_y}(1,0)} \mathcal{V}(p; \tau),  & \sigma(0) \geq 0 \\
   \arg \min\limits_{v_{y}(0,0)} \min\limits_{{v_y}(1,0)} \mathcal{V}(p; \tau), &  \sigma(0) < 0\\
\end{cases}, \\
\overline{R}_{\delta}(\pi, \pi_0; m_{a=1}, \mathcal{V}) &= \delta_{a=1}(\overline{\mathbf{v}}_0, \mathbf{v}_1), \; \text{ where } \overline{\mathbf{v}}_0 \in \argmax_{v_{1}(1,0)} \mathcal{V}(p; \tau).
\end{align*}

\end{lemma}

We now prove $\delta$-regret bounds for $m_a$.

\begin{proof}

Let $a' = 1 -a$. Then $\forall a \in \{ 0, 1\}$ we have that

    \begin{align*}
\delta(\mathbf{v}^*_0, v_1) &= m^*_{a}(\pi) - m^*_{a}(\pi_0) \\
&= p(Y(1) = a \mid T^{\pi} = a) - p(Y(1) = a \mid D^{\pi_0} = a) \\
&= \frac{p(T=a, Y(1) = a)}{p(T = a)} - \frac{p(D=a, Y(1) = a)}{p(D = a)} \\
&= \frac{v_a(a,a') + v_a(a,a)}{\rho_{aa} + \rho_{aa'}} - \frac{v_a(a,a) + v_a(a',a)}{ \rho_{aa} + \rho_{a'a}}\\
&= \frac{(\rho_{a'a} - \rho_{aa'}) \cdot v_a(a,a) + \psi_a(\pi_0) \cdot v_a(a,a') - \psi_a(\pi) \cdot v_a(a',a)}{\psi_0(\pi) \cdot \psi_0(\pi_0)}
\end{align*}
When $a=1$, we have that 
$$
R^*(\pi, \pi_0; m_{a=1}) \leq \frac{(\rho_{01} - \rho_{10}) \cdot v_1(1, 1)  + \psi_1(\pi_0) \cdot \overline{v}_{1}(1, 0) - \psi_1(\pi) \cdot v_1(0, 1)}{\psi_1(\pi) \cdot \psi_1(\pi_0)}
$$

When $a=0$, we have two cases. 

\begin{equation*}
R^*(\pi, \pi_0; m_{a=0}) \leq  \begin{cases}
    \dfrac{(\rho_{10} - \rho_{01}) \cdot \overline{v}_0(0, 0)  + \psi_0(\pi_0) \cdot v_{0}(0, 1) - \psi_0(\pi) \cdot \underline{v}_0(1, 0)}{\psi_0(\pi) \cdot \psi_0(\pi_0)}, &  \rho_{10} > \rho_{01} \\
    \dfrac{(\rho_{10} - \rho_{01}) \cdot \underline{v}_0(0, 0)  + \psi_0(\pi_0) \cdot v_{0}(0, 1) - \psi_0(\pi) \cdot \underline{v}_0(1, 0)}{\psi_0(\pi) \cdot \psi_0(\pi_0)}, &  \rho_{10} \leq \rho_{01}
\end{cases}
\end{equation*}

The lower regret bound is symmetric.


\end{proof}

\newpage
\subsubsection{Baseline regret bounds}\label{appendix:standard_bounds}

We now provide baseline regret bounds for policy performance measures. 
\begin{proposition}[Baseline bounds on $m_u$] Let $m_u(\pi)$ be the expected utility of $\pi$. Then the baseline upper bound on regret is given by

$$
\overline{R}(\pi, \pi_0; m_u, \mathcal{V}) = \sum_{y} u_{0y} \cdot (\overline{v}_y(0,0) + v_y(0,1) - \underline{v}_y(0,0) - \underline{v}_y(1,0)) + u_{1y} \cdot (\overline{v}_y(1,0) - v_y(0,1)),
$$
where the lower bound $\underline{R}(\pi, \pi_0; m_u, \mathcal{V})$ is symmetric.

\end{proposition}

\begin{proof}
By the same argument provided in the proof of Theorem \ref{thm:udelta}, we have that

$$
m^*_u(\pi) = \sum_{t,y} u_{ty} \cdot (v_y(t,0) + v_y(t,1)), \;\; m^*_u(\pi_0) = \sum_{d,y} u_{dy} \cdot (v_y(0,d) + v_y(1,d))
$$
This implies
\begin{align*}
\overline{m}_u(\pi; \mathcal{V}) &= \sum_{y} u_{0y} \cdot (\overline{v}_y(0,0) + v_y(0,1)) + u_{1y} \cdot (\overline{v}_y(1,0) + v_y(1,1))\\
\underline{m}_u(\pi_0; \mathcal{V}) &= \sum_{y} u_{0y} \cdot (\underline{v}_y(0,0) + \underline{v}_y(1,0)) + u_{1y} \cdot (v_y(0,1) + v_y(1,1))
\end{align*}
Simplifying yields the result
\begin{align*}
\overline{R}(\pi, \pi_0; m_u, \mathcal{V}) &= \overline{m}_u(\pi; \mathcal{V}) - \underline{m}_u(\pi_0; \mathcal{V})\\
&= \sum_{y} u_{0y} \cdot (\overline{v}_y(0,0) + v_y(0,1) - \underline{v}_y(0,0) - \underline{v}_y(1,0)) + u_{1y} \cdot (\overline{v}_y(1,0) - v_y(0,1))
\end{align*}

\end{proof}
\begin{proposition}[Baseline bounds on $m_y$]

Let $m^*_{y}(\pi) = p(A^{\pi}=1 \mid Y(1) = y)$ be the positive $(y=1)$ or negative $(y=0)$ class predictive performance of $\pi$. Then the baseline upper regret bound on $m_y$ is given by
$$
\overline{R}(\pi, \pi_0; m_y, \mathcal{V}) = \left(\frac{\overline{v}_y(1,0) + v_y(1,1)}{\underline{v}_y(0,0) + v_y(0,1) + \overline{v}_y(1,0) + v_y(1,1)} \right) - \left(\frac{v_y(0, 1) + v_y(1,1)}{\overline{v}_y(0,0) + v_y(0,1) + \overline{v}_y(1,0) + v_y(1,1)}\right)
$$
where the lower bound $\underline{R}(\pi, \pi_0; m_y, \mathcal{V})$ is symmetric. 
\end{proposition}

\begin{proof}

We have that
$$
m^*_y(\pi) = p(T^\pi=1 \mid Y(1) = y) = \frac{p(T^\pi=1, Y(1) = y)}{p(Y(1) = y)} = \frac{v_y(1,0) + v_y(1,1)}{v_y(0,0) + v_y(0,1) + v_y(1,0) + v_y(1,1)}
$$
And similarly, 
$$
m^*_y(\pi_0) = p(D^{\pi_0}=1 \mid Y(1) = y) = \frac{p(D^{\pi_0}=1, Y(1) = y)}{p(Y(1) = y)} = \frac{v_y(0, 1) + v_y(1,1)}{v_y(0,0) + v_y(0,1) + v_y(1,0) + v_y(1,1)}
$$
Applying the definition of baseline regret yields the result:
\begin{align*}
\overline{R}(\pi, \pi_0; \mathbf{v}, m_y) &= \overline{m}_y(\pi; \mathbf{v}) - \underline{m}_y(\pi_0; \mathbf{v})\\
&= \left(\frac{\overline{v}_y(1,0) + v_y(1,1)}{\underline{v}_y(0,0) + v_y(0,1) + \overline{v}_y(1,0) + v_y(1,1)} \right) - \left(\frac{v_y(0, 1) + v_y(1,1)}{\overline{v}_y(0,0) + v_y(0,1) + \overline{v}_y(1,0) + v_y(1,1)}\right)
\end{align*}
The lower bound follows from applying the same decomposition with $\underline{R}(\pi, \pi_0; \mathbf{v}, m_y) = \underline{m}_y(\pi; \mathbf{v}) - \overline{m}_y(\pi_0; \mathbf{v})$.    
\end{proof}

\begin{proposition}[Baseline bounds on $m_a$]

Let $m^*_{a}(\pi) = p(Y(1) = a \mid A^{\pi}=a)$ be the positive $(y=1)$ or negative $(y=0)$ predictive value of $\pi$. Let $\psi_a(\pi) = p(A^{\pi} = a)$. Then the baseline regret is upper bounded by
\begin{align*}
\overline{R}(\pi, \pi_0; m_{a=1}, \mathcal{V}) &= \frac{\overline{v}_1(1,0) + v_1(1,1)}{\psi_1(\pi)} - \frac{v_1(0,1) + v_1(1,1)}{\psi_1(\pi_0)} \\
\overline{R}(\pi, \pi_0; m_{a=0}, \mathcal{V}) &=  \frac{\overline{v}_0(0,0) + v_0(0,1)}{\psi_0(\pi)} - \frac{\underline{v}_0(0,0) + \underline{v}_0(1,0)}{\psi_0(\pi_0)} \\
\end{align*}
where the lower bounds $\underline{R}(\pi, \pi_0; m_{a=0}, \mathcal{V})$, $\underline{R}(\pi, \pi_0; m_{a=1}, \mathcal{V})$ are symmetric. 

\end{proposition}

\begin{proof}

We begin by showing $\overline{R}(\pi, \pi_0; m_{a=1}, \mathcal{V})$.

$$
    m^*_{a=1}(\pi) = p(Y(1) = 1 \mid T = 1) = \frac{p(Y(1) = 1, T=1)}{p(T=1)} = \frac{v_1(1,0) + v_1(1,1)}{\psi_1(\pi)}
$$

$$
    m^*_{a=1}(\pi_0) = p(Y(1) = 1 \mid D = 1) = \frac{p(Y(1) = 1, D=1)}{p(D=1)} = \frac{v_1(0,1) + v_1(1,1)}{\psi_1(\pi_0)}
$$

\begin{align*}
\overline{R}(\pi, \pi_0; m_{a=1}, \mathcal{V}) &= \overline{m}_{a=1}(\pi; \mathcal{V}) - \underline{m}_{a=1}(\pi_0; \mathcal{V}) \\
 &= \frac{\overline{v}_1(1,0) + v_1(1,1)}{\psi_1(\pi)} - \frac{v_1(0,1) + v_1(1,1)}{\psi_1(\pi_0)}
\end{align*}
The lower bound $\underline{R}(\pi, \pi_0; m_{a=1}, \mathcal{V})$ is symmetric. Similarly, for the upper bound on the negative predictive value $\overline{R}(\pi, \pi_0; m_{a=0}, \mathcal{V})$,
$$
    m^*_{a=0}(\pi) = p(Y(1) = 0 \mid T = 0) = \frac{p(Y(1) = 0, T=0)}{p(T=0)} = \frac{v_0(0,0) + v_0(0,1)}{\psi_0(\pi)}
$$

$$
    m^*_{a=0}(\pi_0) = p(Y(1) = 0 \mid D = 0) = \frac{p(Y(1) = 0, D=0)}{p(D=0)} = \frac{v_0(0,0) + v_0(1,0)}{\psi_0(\pi_0)}
$$
\begin{align*}
\overline{R}(\pi, \pi_0; m_{a=0}, \mathcal{V}) &= \overline{m}_{a=0}(\pi; \mathcal{V}) - \underline{m}_{a=0}(\pi_0; \mathcal{V}) \\
 &= \frac{\overline{v}_0(0,0) + v_0(0,1)}{\psi_0(\pi)} - \frac{\underline{v}_0(0,0) + \underline{v}_0(1,0)}{\psi_0(\pi_0)}
\end{align*}
The lower bound $\underline{R}(\pi, \pi_0; m_{a=0}, \mathcal{V})$ is symmetric.

\end{proof}

\newpage
\subsubsection{Proof of Theorem \ref{thm:delta_seperation}}\label{subsec:seperation_proof}
\begin{proof}

We will bound $\overline{\Delta}(m, \mathcal{V}) = \overline{R}(\pi, \pi_0; m, \mathcal{V}) - \overline{R}_{\delta}(\pi, \pi_0; m, \mathcal{V})$. The result over the full interval follows by taking $\Delta(m, \mathcal{V}) = 2 \cdot \overline{\Delta}(m, \mathcal{V})$ because $\overline{R}(\pi, \pi_0; m, \mathcal{V}) - \overline{R}_{\delta}(\pi, \pi_0; m, \mathcal{V}) = \underline{R}_{\delta}(\pi, \pi_0; m, \mathcal{V}) - \underline{R}(\pi, \pi_0; m, \mathcal{V})$. We begin for showing the result for the positive and negative class predictive performance. 

Case 1: $\overline{v}_{10} > v_{01}$.

\begin{align*}
\overline{\Delta}(m_y, \mathcal{V}) &= \overline{R}(\pi, \pi_0; m_y, \mathcal{V}) - \overline{R}_{\delta}(\pi, \pi_0; m_y, \mathcal{V}) \\
&= \left(\frac{v_{11} + \overline{v}_{10}}{\underline{v}_{00} + \overline{v}_{10} + v_{01} + v_{11}} - \frac{v_{11} + v_{01}}{\overline{v}_{00} + \overline{v}_{10} + v_{01} + v_{11}}\right)
- \frac{\overline{v}_{10} - v_{01}}{ \underline{v}_{00} + \overline{v}_{10} + v_{01} + v_{11}}\\
&= \frac{(v_{11} + v_{01}) \cdot (\overline{v}_{00} - \underline{v}_{00})}{ (\underline{v}_{00} + \overline{v}_{10} + v_{01} + v_{11}) \cdot (\overline{v}_{00} + \overline{v}_{10} + v_{01} + v_{11})} \\
&\geq \frac{\alpha \cdot v_{11} }{ (\overline{v}_{00} + \overline{v}_{10} + v_{01} + v_{11}) \cdot (\overline{v}_{00} + \overline{v}_{10} + v_{01} + v_{11})} \\
&= \frac{\alpha \cdot v_{11}}{(\overline{\gamma})^2}
\end{align*}

Case 2: $\overline{v}_{10} \leq v_{01}$.

\begin{align*}
\overline{\Delta}(m_y, \mathcal{V}) &= \overline{R}(\pi, \pi_0; m_y, \mathcal{V}) - \overline{R}_{\delta}(\pi, \pi_0; m_y, \mathcal{V}) \\
&= \left(\frac{v_{11} + \overline{v}_{10}}{\underline{v}_{00} + \overline{v}_{10} + v_{01} + v_{11}} - \frac{v_{11} + v_{01}}{\overline{v}_{00} + \overline{v}_{10} + v_{01} + v_{11}}\right)
- \frac{\overline{v}_{10} - v_{01}}{ \overline{v}_{00} + \overline{v}_{10} + v_{01} + v_{11}}\\
&= \frac{(v_{11} + \overline{v}_{10}) \cdot (\overline{v}_{00} - \underline{v}_{00})}{ (\underline{v}_{00} + \overline{v}_{10} + v_{01} + v_{11}) \cdot (\overline{v}_{00} + \overline{v}_{10} + v_{01} + v_{11})} \\
&\geq \frac{\alpha \cdot v_{11} }{(\overline{\gamma})^2}
\end{align*}

Next we will bound $\overline{\Delta}(m_{a=0}, \mathcal{V})$. Case 1: $\rho_{10} - \rho_{11} > 0$. Note that $\psi_0(\pi) = \rho_{01} + \rho_{00}$ and $\psi_0(\pi_0) = \rho_{10} + \rho_{00}$. We have that
\begin{align*}
\overline{\Delta}(m_{a=0}, \mathcal{V}) &= \overline{R}(\pi, \pi_0; m_{a=0}, \mathcal{V}) - \overline{R}_\delta(\pi, \pi_0; m_{a=0}, \mathcal{V}) \\
&= \left( \frac{\overline{v}_0(0,0) + v_0(0,1)}{\rho_{01} + \rho_{00}} - \frac{\underline{v}_0(0,0) + \underline{v}_0(1,0)}{\rho_{10} + \rho_{00}} \right) \\
& \;\;\;\;\;\; - \frac{(\rho_{10} - \rho_{01}) \cdot \overline{v}_0(0,0) + (\rho_{10} + \rho_{00}) \cdot v_{0}(0, 1) - (\rho_{01} + \rho_{00}) \cdot \underline{v}_0(1, 0)}{(\rho_{01} + \rho_{00}) \cdot (\rho_{01} + \rho_{00})} \\
&= \frac{(\overline{v}_0(0,0) - \underline{v}_0(0,0))}{(\rho_{01} + \rho_{00})}\\
&= \frac{\alpha}{\psi_0(\pi)}.
\end{align*}
The third equality follows from finding a common denominator and simplifying. Case 2: $\rho_{10} - \rho_{11} \leq 0$. Following the same argument, we have that $\overline{\Delta}(m_{a=0}, \mathcal{V}) = \frac{\alpha}{\psi_0(\pi_0)}$. Thus
$$
\overline{\Delta}(m_{a=0}, \mathcal{V}) \geq \min\{\frac{\alpha}{\psi_0(\pi)}, \frac{\alpha}{\psi_0(\pi_0) } \} = \frac{\alpha}{\max\{\psi_0(\pi), \psi_0(\pi_0)\}}.
$$

Next we will bound $\overline{\Delta}(m_u, \mathcal{V})$. Let $\lambda_{ay} = u_{ay} - u_{a'y}$.

Case 1: $\lambda_{11} > \lambda_{10}$.
\begin{align*}
\overline{\Delta}(m_u, \mathcal{V}) &= \overline{R}(\pi, \pi_0; m_u, \mathcal{V}) - \overline{R}_{\delta}(\pi, \pi_0;  m_u, \mathcal{V}) \\
&= (u_{00} + u_{01}) (\overline{v}_0(0,0) - \underline{v}_0(0,0)) + (u_{00} + u_{11}) (\overline{v}_0(1,0) - \underline{v}_0(1,0))
\end{align*}

Case 2: $\lambda_{11} \leq \lambda_{10}$.
\begin{align*}
\overline{\Delta}(m_u, \mathcal{V}) &= \overline{R}(\pi, \pi_0; m_u, \mathcal{V}) - \overline{R}_{\delta}(\pi, \pi_0;  m_u, \mathcal{V}) \\
&= (u_{00} + u_{01}) (\overline{v}_0(0,0) - \underline{v}_0(0,0)) + (u_{10} + u_{01}) (\overline{v}_0(1,0) - \underline{v}_0(1,0))
\end{align*}

Combining cases yields the result: 
$$
\overline{\Delta}(m_u, \mathcal{V}) \leq (u_{00} + u_{01}) \cdot (\overline{v}_0(0,0) - \underline{v}_0(0,0))
$$

The result that $\overline{\Delta}(m_{a=1}, \mathcal{V}) = \overline{R}(\pi, \pi_0; m_{a=1}, \mathcal{V}) - \overline{R}_\delta(\pi, \pi_0; m_{a=1}, \mathcal{V}) = 0$ follows directly from plugging in definitions of baseline upper regret bound and $\delta$ regret bound and simplifying. 

\end{proof}

\newpage
\section{Assumption Mapping Extensions and Proofs}\label{appendix:assumption_extensions}

In this appendix, we discuss additional causal assumptions which imply uncertainty sets over partially-identified $v$-statistics. Rosenbaum's $\Gamma$-sensitivity model \citep{rosenbaum2005sensitivity}, the proximal identification framework \citep{ghassami2023partial}, and Manski-style no assumptions bounds \citep{manski1998monotone} imply bounding functions $\tau(\cdot)$, which can be used to construct $\mathcal{V}(p;\tau)$ by invoking Lemma \ref{lemma:assumption_mapping}.

\subsection{Rosenbaum's $\Gamma$-Sensitivity Model}

Rosenbaum's \textit{$\Gamma$-sensitivity analysis model} bounds the influence of unobserved confounders on the odds of being treated versus untreated \citep{rosenbaum1987sensitivity}. \citet{namkoong2020off} leverage a sequential adaptation of this model to partially identify the value function of a new policy $\pi$ given confounded off-policy data, while \citet{zhang2020effect} leverage this model to rank individualized treatment rules under confounding. 

\begin{assumption}[$\Gamma$-sensitivity]\label{assumption:rosenbaum} For some $\Gamma \geq 1$,  $(D, T, X, U, Y(1)) \sim p^*(\cdot)$ satisfies

\begin{equation}
\Gamma^{-1} \leq \frac{P(D=1 \mid X, U=u)}{P(D=0 \mid X, U=u)} \frac{P(D=0 \mid X, U=\tilde{u})}{P(D=1 \mid X, U=\tilde{u})} \leq \Gamma
\end{equation}

for all $u, \tilde{u} \in \mathcal{U}$ and $X \in \mathcal{X}$ with probability one.
    
\end{assumption}

\begin{lemma}[\citet{rambachan2022counterfactual}]

Suppose that $(D, T, X, U, Y(1)) \sim p^*(\cdot)$  satisfies Assumption \ref{assumption:rosenbaum} for some $\Gamma \geq 1$. Then 
$$
\Gamma^{-1} \cdot \mu_1(x) \leq \mu_0(x) \leq \Gamma \cdot \mu_1(x), \;\; \forall x \in X.
$$

\end{lemma}

This result follows from Proposition 8.2 of \citet{rambachan2022counterfactual}, which shows that Rosenbaum’s $\Gamma$-sensitivity model implies a marginal sensitivity model in binary outcome settings. Because Rosenbaum's $\Gamma$ sensitivity model does not imply \textit{sharp} bounds under the MSM, an uncertainty set constructed using $\underline{\tau}(x) = \Gamma^{-1} \cdot \mu_1(x)$, $\overline{\tau}(x) = \Gamma \cdot \mu_1(x)$ does not guarantee we recover the tightest regret interval attainable under Assumption \ref{assumption:rosenbaum}.

\subsection{Marginal Sensitivity Model}

The marginal sensitivity model (MSM) restricts the extent to which unobserved confounders impact the odds of treatment under the status quo policy \citep{tan2006distributional}. In particular, this model holds that a confounder $U \in \mathbb{R}^{k}$ exists such that decisions would be conditionally randomized if we controlled for both $X$ and $U$. As pointed out by \citet{robins2002covariance}, it suffices to assign $U=Y(1)$.

\begin{assumption}[Marginal Sensitivity Model]\label{assumption:msm}

For some $\Lambda \geq 1$, $(X, D, Y(1)) \sim p^*(\cdot)$ satisfies
\begin{equation}
\left. \Lambda^{-1} \leq \frac{p(D=1 \mid X, Y(1))}{p(D=0 \mid X, Y(1))} \cdot \frac{p(D=1 \mid X)}{p(D=0 \mid X)} \leq \Lambda \right.
\end{equation}

\end{assumption}

As in the IV setting, Assumption \ref{assumption:msm} implies bounds on the unobserved regression $\mu_0(x)$, which in turn can be used to construct $\tau(\cdot)$.

\begin{lemma}[\citet{rambachan2022counterfactual}]\label{lemma:msm} Suppose Assumptions \ref{assumption:consistency} and \ref{assumption:positivity} hold and that the MSM (\ref{assumption:msm}) is satisfied for some $\Lambda \geq 1$. Then 
$$
 \Lambda^{-1} \cdot \mu_1(x) \leq \mu_0(x) \leq \Lambda \cdot \mu_1(x), \; \forall x \in X.
$$

\end{lemma}

Lemma \ref{lemma:msm} follows from Bayes' rule, and is a direct consequence of Proposition 8.1 in \cite{rambachan2022counterfactual}. Moreover, because Rosenbaum's $\Gamma$ model implies bounds under the MSM \citep{rambachan2022counterfactual}, we can similarly construct bounding functions under this model by taking $\underline{\tau}(x) = \Gamma^{-1} \cdot \mu_1(x)$, $\overline{\tau}(x) = \Gamma \cdot \mu_1(x)$.

\begin{figure}[!htbp]
\parbox{.25\textwidth}{
	\vspace{0.9cm}
	\centering
	\begin{tikzpicture}[->,>=stealth',node distance=1cm, auto,]
		\node (Z) {$Z$};
		\node[right = of Z] (X) {$X$};
		\node[est, dashed, above = of X, yshift=-5mm] (U) {$U$};
		\node[below = of Z] (D) {$D$};
		\node[right = of D, xshift=1.5cm] (Y) {$Y$}; 
		\draw (U) -- (X);
		\draw (X) -- (Y);
		\draw (Z) -- (D);
		\draw (D) -- (Y);
            \draw (U) -- (D);
            \draw (X) -- (D);
            \draw (X) -- (Z);
            \draw (U) -- (Y);
	\end{tikzpicture}
	\quad \\ \bigskip (a) A DAG containing an instrumental variable $Z$.
}
\hfill
\parbox{.25\textwidth}{
	\vspace{0.9cm}
	\centering
	\begin{tikzpicture}[->,>=stealth',node distance=1cm, auto,]
		\node  (Z) {$Z$};
		\node[right = of Z] (X) {$X$};
		\node[est, dashed, above = of X, yshift=-5mm] (U) {$U$};
		\node[below = of Z] (D) {$D$};
		\node[right = of D, xshift=1.5cm] (Y) {$Y$}; 
		\draw (U) -- (Z);
		\draw (U) -- (X);
		\draw (X) -- (Y);
		\draw (Z) -- (D);
		\draw (D) -- (Y);
            \draw (U) -- (D);
            \draw (X) -- (D);
            \draw (U) -- (Y);
            \draw (X) -- (Z);
	\end{tikzpicture}\
	\quad \\ \bigskip (b) A DAG containing a treatment confounding proxy $Z$.
}
\hfill
\parbox{.25\textwidth}{
	\vspace{0.9cm}
	\centering
	\begin{tikzpicture}[->,>=stealth',node distance=1cm, auto,]
		\node (Z) {$Z$};
		\node[ right = of Z] (X) {$X$};
		\node[est,dashed, above = of X, yshift=-5mm] (U) {$U$};
		\node[below = of Z] (D) {$D$};
		\node[right = of D, xshift=1.7cm] (Y) {$Y$}; 
            \node [right = of X] (W) {$W$};
		\draw (U) -- (Z);
		\draw (U) -- (X);
		\draw (X) -- (Y);
		\draw (Z) -- (D);
		\draw (D) -- (Y);
            \draw (U) -- (D);
            \draw (X) -- (D);
            \draw (U) -- (Y);
            \draw (X) -- (Z);
            \draw (X) -- (W);
            \draw (Z) -- (Y);
            \draw (D) -- (W);
            \draw (W) -- (Y);
            \draw (U) -- (W);
	\end{tikzpicture}
	\quad \\ \bigskip (c) A DAG containing a treatment confounding proxy $Z$ and outcome confounding proxy $W$.
}
\caption{Three sets of structural assumptions on $p^*(\cdot)$ which imply bounding functions $\tau(\cdot)$.} \label{fig:assumptions}
\end{figure}
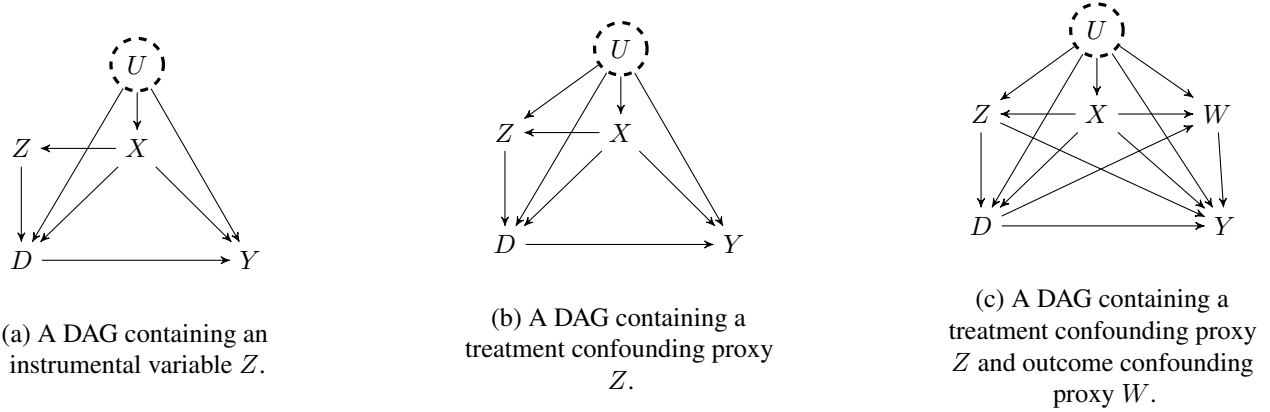

\newpage
\subsection{Instrumental Variable}

Instrumental variables can be used for identification when the status quo policy was influenced by a random source of selection rate heterogeneity \citep{lakkaraju2017selective, chen2023learning, rambachan2021identifying, kleinberg2018human}. For example, in medical testing contexts, an instrument is available when patients are randomly assigned to physicians with heterogeneous testing rates \citep{abaluck2016determinants}.

\begin{assumption}[Instrumental Variable] Let $Z$ be a finite-valued instrument satisfying
\begin{enumerate}
    \item Relevance: $Z \nCI D \mid X$
    \item IV independence: $Z \CI U$ 
    \item Exclusion restriction: $Z \CI Y \mid D, X, U$  
\end{enumerate}\label{assumption:instrument}

\end{assumption}

In the following Lemma, we show that Assumption \ref{assumption:instrument} implies sharp bounds on the unobserved outcome regression. 
\begin{lemma}[\citet{rambachan2022counterfactual, rambachan2021identifying}]\label{lemma:iv_mapping} Suppose Assumptions \ref{assumption:consistency}, \ref{assumption:positivity}, and \ref{assumption:instrument} hold. Let $\mu(x) = \bE[Y(1) \mid X=x]$ and $\mu_d(x) = \bE[Y(1) \mid D=d, X=x]$ be outcome regressions and let $e_d(x) = p(D=d \mid X=x)$ be the propensity function. Then $\forall x \in X$, $\mu_0(x)$ is bounded by
\begin{equation*}
    \frac{\underline{\mu}(x)-\mu_1(x) \cdot e_1(x)}{e_0(x)} \leq \mu_0(x) \leq \frac{\overline{\mu}(x) - \mu_1(x) \cdot e_1(x)}{e_0(x)},
\end{equation*}
where $\underline{\mu}(x)= \max_{\tilde{z} \in \mathcal{Z}}\left\{\mu_1(x, \tilde{z}) \cdot e_1(x, \tilde{z})\right\}$, $\bar{\mu}(x)=\min _{\tilde{z} \in \mathcal{Z}}\left\{ e_0(x, \tilde{z}) + \mu_1(x, \tilde{z}) \cdot e_1(x, \tilde{z})\right\}$.

\end{lemma}

\begin{proof}
    By iterated expectations, $\mu(x) = \mu_1(x) \cdot e_1(x) + \mu_0(x) \cdot e_0(x)$. Therefore
    \begin{align}
        \mu_1(x) \cdot e_1(x) &\leq \mu(x) \leq e_0(x) + \mu_1(x)  \cdot e_1(x) \label{eq:ia}\\
        \mu_1(x,z) \cdot e_1(x,z) &\leq \mu(x) \leq e_0(x,z) + \mu_1(x,z)  \cdot e_1(x,z) \label{eq:ib} \\
         \max_{\tilde{z} \in \mathcal{Z}}\left\{\mu_1(x, \tilde{z}) \cdot e_1(x, \tilde{z})\right\} &\leq \mu(x) \leq \min_{\tilde{z} \in \mathcal{Z}}\left\{ e_0(x, \tilde{z}) + \mu_1(x, \tilde{z}) \cdot e_1(x, \tilde{z}) \right\} \label{eq:ic}
    \end{align}

    where (\ref{eq:ia}) implies sharp bounds on $\mu_0(x)$ (\citet{manski1998monotone}) and (\ref{eq:ib}) follows from Assumption \ref{assumption:instrument}. Solving for $\mu_0(x)$ yields

    \begin{equation*}
        \underline{\mu}_{0}(x) = \frac{\underline{\mu}(x)-\mu_1(x) \cdot e_1(x)}{e_0(x)}, \;\;\;\;
        \overline{\mu}_{0}(x) = \frac{\overline{\mu}(x) - \mu_1(x) \cdot e_1(x)}{e_0(x)}
    \end{equation*}

    with $\underline{\mu}(x)$, $\overline{\mu}(x)$ defined as in (\ref{eq:ic}). Note that both terms are defined because $e_0(x) > 0, \forall x \in X$ by Assumption \ref{assumption:positivity}. 
\end{proof}

\subsection{Proximal Variable}

The proximal causal inference framework relaxes the IV unconfoundedness condition imposed by the IV framework \citep{tchetgen2020introduction}. \citet{ghassami2023partial} extend this framework to support partial-identification of the average treatment effect. We show that this framework also implies point-wise bounding functions on the unobserved outcome regression. We discuss two versions of this framework -- the treatment confounded proxy (Fig \ref{fig:assumptions}.b) and treatment/outcome confounded proxy (Fig \ref{fig:assumptions}.c) -- which are most realistic in our setting. 

\begin{assumption}(Treatment confounding proxy)\label{assumption:pt_proxy} Let $Z$ be a treatment confounding proxy variable such that $(Z, Y, D, X, U) \sim p^*(\cdot)$ satisfies $Z \CI Y \mid D, X, U$.
\end{assumption}

Assumption \ref{assumption:pt_proxy} is identical to the exclusion restriction in the IV setting, but allows for a relaxation of IV independence. Therefore, this condition can be reasonable in settings where IV independence is violated (e.g., confounded assignment of instances to decision-makers). 

\begin{assumption}[Treatment confounding bridge]\label{assumption:pt_bridge} There exists a non-negative bridge function $m$ such that almost surely 
$$
\mathbb{E}[m(Z, D, X) \mid D, X, U]=\frac{p(U \mid 1-D, X)}{p(U \mid D, X)}.
$$
\end{assumption}

Note that this condition only stipulates that $m$ exists and does not require its point identification. 

\begin{lemma}\label{lemma:pt_mapping}(\citet{ghassami2023partial}) Let Assumptions \ref{assumption:positivity}, \ref{assumption:pt_proxy} and \ref{assumption:pt_bridge} hold. Let $\mu_{d}(x,z) = \bE[Y(1) \mid D=d, Z=z, X=x]$. Then $\forall x \in X$
\begin{equation}
    \min_{\tilde{z} \in Z} \{ \mu_1(x, \tilde{z})\} \leq \mu_0(x) \leq \max_{\tilde{z} \in Z} \{ \mu_1(x, \tilde{z})\}
\end{equation}
\end{lemma}

\begin{assumption}(Treatment and outcome confounding proxy)\label{assumption:pto_proxy} Let $Z$ be a treatment confounding proxy and $W$ be an outcome confounding proxy such that $(Z, W, Y, D, X, U) \sim p^*(\cdot)$ satisfies $Z \CI W \mid D, X, U$.
\end{assumption}

Assumption \ref{assumption:pto_proxy} enables a relaxation of the IV unconfoundedness condition. This assumption also obviates the exclusion restriction required by the IV (\ref{assumption:instrument}) and treatment confounding proxy (\ref{assumption:pt_proxy}) assumptions. Therefore, Assumption \ref{assumption:pto_proxy} applies given a confounded treatment assignment variable (e.g., the identity of a decision-maker $Z$) and a temporally-lagged outcome variable $W$. Temporally-lagged outcomes are ubiquitous in multi-step decision-making processes, where the final outcome of interest is often preceded by a series of intermediary decisions.

\begin{assumption}(Outcome confounding bridge)\label{assumption:o_bridge}
    There exists a non-negative bridge function $h$ such that almost surely $\bE[Y \mid D, X, U] = \bE[h(W, D, X)  \mid D, X, U]$
    
\end{assumption}

Intuitively, this assumption requires that $W$ is sufficiently informative of confounders such that there exists a function of $W$ which recovers the potential outcome mean as a function of $U$ \citep{miao2018confounding, ghassami2023partial}. Similarly to Assumption \ref{assumption:pt_bridge}, this condition does not require identification of $h$.

\begin{lemma}[\citet{ghassami2023partial}]\label{lemma:pp}
    Let Assumptions \ref{assumption:positivity}, \ref{assumption:pt_bridge}, \ref{assumption:pto_proxy}, and \ref{assumption:o_bridge} hold. Let $\mu_1(x) = \bE[Y(1) \mid D=1, X=x]$, $\eta_1(w) = p(W=w \mid D=1, X=x)$, $\eta_1(z) = p(Z=z \mid D=1, X=x)$, $\eta_1(w,z) = p(W=w, Z=z \mid D=1, X=x)$. Then 
\begin{equation*}
\mu_1(x) \cdot \min_{w,z} \frac{ \eta_1(w,z) }{\eta_1(w) \cdot \eta_1(z) } \leq \mu_0(x) \leq  \mu_1(x) \cdot \max_{w,z} \frac{ \eta_1(w,z) }{\eta_1(w) \cdot \eta_1(z) }, \;\; \forall x \in X.
\end{equation*}
\end{lemma}

\subsection{Manski Style No-Assumptions Bounds}

Manski-style no assumptions bounds on partially identified policy comparison terms \citep{manski1989anatomy, manski1998monotone} follow by setting  $\underline{\tau}(x) = 0$, $\overline{\tau}(x) = 1$. Invoking Lemma \ref{lemma:assumption_mapping} yields $\mathcal{H}(v_y(t, 0); \tau) = [0, p(D=0, T=1)]$, which is the same interval recovered by the \textit{unconstrained} uncertainty set $\mathcal{V}(p;\tau)$.

\subsection{Assumption Mapping Proofs}\label{appendix:proofs}

\subsubsection{Proof of Lemma \ref{lemma:assumption_mapping}.}

\begin{proof} For the upper bound $\overline{\tau}(x)$ we have that $\forall x \in X$
\begin{align}
\bE[Y(1) \mid D=0, X=x] &\leq \overline{\tau}(x) \nonumber \\
\bE[Y(1) \mid D=0, T=t, X=x] &\leq \overline{\tau}(x) \label{eq:gena}\\
\bE[Y(1) \mid D=0, T=t, X=x] \cdot p(D=0, T=t, X=x) &\leq  p(D=0, T=t, X=x) \cdot \overline{\tau}(x) \nonumber
\end{align}
where (\ref{eq:gena}) holds because $T \CI \{D, Y(1) \} \mid X$. Simplifying the right hand side yields 
\begin{align*}
&\bE[Y(1) \mid D=0, T=t, X=x] \cdot p(D=0, T=t, X=x) \\
&\leq  p(D=0 \mid T=t, X=x) \cdot p(T=t \mid X=x) \cdot p(X=x) \cdot \overline{\tau}(x) \\
&\leq  p(D=0 \mid X=x) \cdot p(T=t \mid X=x) \cdot p(X=x) \cdot \overline{\tau}(x) 
\end{align*}
The result follows from marginalizing over $X$
\begin{align*}
v_1(t,0) &\leq \sum_x  e_0(x) \cdot \pi_t(x) \cdot \overline{\tau}(x) \cdot p(X=x)  \\
        &= \bE[\overline{\tau}(x) \cdot e_0(x) \cdot \pi_t(x)]
\end{align*}

The upper bound is sharp because equality holds when $\overline{\tau}(x) = \bE[Y(1) \mid D=0, X=x] \; \forall x \in X$. The lower bound on $v_t(t,0)$ follows by the same argument. 
\end{proof}

\subsubsection{Proof of Theorem \ref{thm:minimality}.}

Importantly, Theorem \ref{thm:minimality} does \textit{not} imply that $\mathcal{V}(p;\tau)$ guarantees the tightest regret interval attainable from a specific causal assumption. In some cases, it may be possible to contract regret intervals further by exploiting additional information provided by a causal assumption. We discuss an example involving Rosenbaum's $\Gamma$ model above. Nevertheless, the flexibility of our assumption mapping approach is important for the real-world applicability of our framework. 

\begin{proof}

$\mathcal{V}(p; \tau)$ is minimal if (1) it is consistent with $p(X, D, T, Y)$ and  $\underline{\tau}, \overline{\tau}$ and (2) no consistent set $\mathcal{V}^*(p; \tau)$ exists such that $\lambda(\mathcal{V}^*(p; \tau)) < \lambda(\mathcal{V}(p; \tau))$. The first condition holds if $\forall \mathbf{v}_0 \in \mathcal{V}(p; \tau)$, 
\begin{align*}
v_1(0,0) &\in [\bE[\underline{\tau}(x) \cdot e_0(x) \cdot \pi_0(x)], \bE[\overline{\tau}(x) \cdot e_0(x) \cdot \pi_0(x)]]\\
v_1(1,0) &\in [\bE[\underline{\tau}(x) \cdot e_0(x) \cdot \pi_1(x)], \bE[\overline{\tau}(x) \cdot e_0(x) \cdot \pi_1(x)]]
\end{align*}

and $v_0(t,d) +  v_1(t,d) = \rho_{td}$, $\forall t, d$. This follows by definition of $\mathcal{V}(p; \tau)$. Additionally, 
$$
\sum_{d,t} p(D=d, T=t) = 1 \implies \sum_{d,t} v_0(d,t) + v_1(d,t) =1 \implies \sum_{d,t,y} v_y(d,t) = 1
$$

where the first equality follows by the law of total probability and the first implication follows by the constraint. Thus

$$
\sum_{d,y} v_y(t,d) = p(T=t), \; \sum_{t,y} v_y(t,d) = p(D=d), \; p(Y=y) = \sum_{t,y} v_y(t,1) \leq \sum_{d,t,d} v_y(t,d) \leq 1.
$$

Therefore, the constraint $v_0(t,d) + v_1(t,d) = 1 \; \forall t,d$ implies that $\mathbf{v}_0$ is consistent with other marginals of $p(X, D, T, Y)$. We will show (2) by contradiction. The Lebesgue measure over a cartesian product of intervals is given by  
$\lambda(\mathcal{V}) = \prod_{y,t} |\overline{v}_y(t,0) - \underline{v}_y(t,0)|$. Let $\mathcal{V}^*(p;f) \in [0,1]^4$ be an arbitrary set satisfying $\lambda(\mathcal{V}) > \lambda(\mathcal{V}^*)$. By a property of Lebesgue measures it follows that $A \supseteq B \implies \lambda(A) \leq \lambda(B)$. Thus  
\begin{align*}
\lambda(\mathcal{V}) > \lambda(\mathcal{V}^*) &\implies \; \mathcal{V} \not\subseteq \mathcal{V}^* \\
&\implies \exists \; \mathbf{v}_0 \in \mathcal{V} \text{ s.t. } \mathbf{v}_0 \not \in \mathcal{V}^*\\
&\implies \exists \; t \in \{ 0,1 \} \text{ s.t. } \overline{v}_1(t,0) > \overline{v}_1^*(t,0) \text{ or } \underline{v}^*_1(t,0) > \underline{v}_1(t,0).
\end{align*}

Define $s(x) = p(x) \cdot e_0(x) \cdot  \pi_t(x)$. Then
\begin{align*}
\overline{v}_1(t,0) > \overline{v}_1^*(t,0) &\implies \sum_x s(x) \cdot \overline{\tau}(x) > \sum_x s(x) \cdot \overline{\tau}^*(x) \\
&\implies \sum_x s(x) \cdot \overline{\tau}(x) - \sum_x s(x) \cdot \overline{\tau}^*(x) > 0 \\
&\implies \sum_x s(x) (\overline{\tau}(x) - \overline{\tau}^*(x)) > 0\\
&\implies \exists \; x \in X \text{ s.t. } \overline{\tau}^*(x) < \overline{\tau}(x) \\
&\implies \exists \; x \in X \text{ s.t. } \overline{\tau}^*(x) \leq \bE[Y(1) \mid D=0, X=x] \leq \overline{\tau}(x)
\end{align*}

However, the final implication violates the condition imposed by the bounding function. An analogous violation occurs for the lower bound $\underline{v}^*_1(t,0) > \underline{v}_1(t,0)$, proving the contradiction. 

\end{proof}
\newpage
\section{Theoretical Estimation Results}\label{appendix:estimation_results}

\subsubsection{Proof of Theorem \ref{thm:pl_convergence}}
\begin{proof}

Let $P_n(f)$ denote sample averages $\frac{1}{n}\sum_{i=1}^n f(O_i)$ on a separate fold of the data from that used to estimate $\hat f$. We can decompose the error term into 

$$
\hat{v} - v = \underbrace{\left(P_n-P\right) f}_{Z^*} + \underbrace{\left(P_n-P\right)(\widehat{f}-f)}_{T_1} + \underbrace{P(\widehat{f}-f)}_{T_2}.
$$

$Z^{*} = O_p(1 / \sqrt{n})$ by the central limit theorem. $T_1 = o_p(1/\sqrt{n})$ by consistency of $\hat{f}$ in the $L_2$ norm and sample splitting \citep{kennedy2020sharp}. Finally, we have for $T_2$ that 
\begin{align*}
P(\hat{f} - f)& = P\left(\hat{e}(X) \cdot \hat{\tau}(X) - e(X) \cdot\tau(X)  \right) \\
&= P\left(\hat{e}(X) \cdot \hat{\tau}(X) + \hat{e}(X) \cdot \tau(X) - \hat{e}(X) \cdot \tau(X) - e(X) \cdot\tau(X)  \right)\\
&= P\left(\hat{e}(X)(\hat{\tau}(X) - \tau(X)) + \tau(X) \cdot (\hat{e}(X) - e(X)) \right) \\
&= P\left(\hat{e}(X)(\hat{\tau}(X) - \tau(X))) + P(\tau(X) \cdot (\hat{e}(X) - e(X)) \right) \\
&\leq ||\hat{e}(x) - e(x)|| + ||\hat{\tau}(x) - \tau(x)||
\end{align*}

where the first equality follows by adding and subtracting $\hat{e}(X) \cdot \tau(X)$ and the final inequality follows by Cauchy-Schwarz. The overall convergence rate follows by combining terms. 

\end{proof}

\subsection{Rate of doubly robust estimator}

\begin{theorem}
    The doubly-robust estimator satisfies
    \begin{align*}
    \hat {\bar{v}}_{DR} - \bar{v} = O_P\Big( \norm{e(x) - \hat e(x)}\norm{\mu_1(x) - \hat \mu_1(x)}  \Big) + (P_n - P) \phi(O; \eta) + o_p(\frac{1}{\sqrt{n}})
    \end{align*}

    when $\norm{\phi - \hat \phi} = o_p(1)$.
\end{theorem}

This theorem demonstrates that the error of our estimator is a product of nuisance function errors. This enables us to achieve faster rate of convergence even when estimating the nuisance function at slow rates. For example, to obtain $n^{-1/2}$ rates for our estimator, it is sufficient to estimate the nuisance functions at $n^{-1/4}$, allowing us to use flexible machine learning methods to nonparametrically estimate the nuisance functions under smoothness or sparsity assumptions.
The theorem's condition that $\phi$ converges in probability in the $L_2(P)$ norm is mild and can be satisfied by using flexible regression methods.

\begin{proof}
Let $\eta$ indicate the nuisance functions $(e, \pi, \mu_1)$ and let $\hat  \eta $ indicate $(\hat e,  \pi, \hat \mu_1)$.
\begin{align*}
    P_n(\phi(O; \hat \eta)) - \bar{v} = \overbrace{P(\phi(O; \hat \eta) - \bar{v} )}^{A} + \overbrace{(P_n-P)(\phi(O; \hat \eta) - \phi(O; \eta)}^{B} + \overbrace{(P-P_n)(-\phi(O; \eta))}^{C}
\end{align*}    

For term A, we have that 
\begin{align*}
    P(\phi(O; \hat \eta) - \bar{v} ) =  \int  &\bar \mu(x) \Big(e(x) - \bar{e}(x)\Big)\Big(\bar \pi(x) - \pi(x)\Big) + \bar e(x) \Big(\pi(x) - \bar{\pi}(x)\Big)\Big(\bar \mu(x) - \mu(x)\Big) \\
&+ \Big(\mu(x) - \bar \mu (x)\Big) \Big(e(x) - \bar{e}(x)\Big)\Big(\bar \pi(x) - \pi(x)\Big))dP \\
&= O_P\Big( \norm{e(x) - \hat e(x)}\norm{\mu(x) - \hat \mu(x)}  \Big)
\end{align*}

where the second line applies Cauchy-Schwarz.

For term B, since $P_n$ is the empirical measure on an independent sample from $\hat{P}$, we can apply Lemma 2 of \citet{kennedy2020sharp} with our assumption that $\norm{\phi - \hat \phi} = o_{P}(1)$:
\begin{align*}
    (P_n-P)(\phi(O; \hat \eta) - \phi(O; \eta) = O_P\Big( \frac{\norm{\phi(O; \hat \eta) - \phi(O; \eta)}}{\sqrt{n}}\Big) = o_P(\frac{1}{\sqrt{n}})
\end{align*}  

Combining yields the result.
\end{proof}

\newpage
\section{Experiment Setup Details and Further Results}\label{appendix:experiments}

\subsubsection{Synthetic setup details}

We sample data from the probability functions 
\begin{align*}
\gamma(X_i) &:= \sigma'(X_i \times W_{z} ), \pi(X_i) := \sigma(X_i \times W_{\pi})\\
\pi_0(V_i, Z_i) &:= \sigma(V_i \times W_{\pi_0} + \beta_0 \cdot Z_i) \\
\mu_1(V_i, Z_i) &:= \sigma(V_i \times W_{\mu_1} + \beta_1 \cdot Z_i) \label{eq:dgp_iv_exclusion} \\
\mu_0(V_i, Z_i) &:= \begin{cases}
   \Lambda^* \cdot \mu_1(V_i), \Lambda^* \in U(\Lambda^{-1}, \Lambda) & \text { (MSM) } \\
   \sigma(V_i \times W_{\mu_0} + \beta_1 \cdot Z_i) & \; \text{ (IV) } \\
\end{cases} \\
\mu(V_i, Z_i) &:= \mu_1(V_i, Z_i)  \pi_1(V_i, Z_i) + \mu_0(V_i, Z_i)  \pi_0(V_i, Z_i)
\end{align*}
where $\sigma(x) = \frac{1}{1+e^{-x}}$ and $\sigma'(x)$ is the softmax function. We then sample $Z_i \sim \text{Multinomial}(\gamma(X_i))$, $D^{\pi_0}_i \sim \text{Bern}(\pi_0(V_i, Z_i))$, $T^{\pi}_i \sim \text{Bern}(\pi(X_i))$, and outcomes
\begin{align}
    Y_i(1) &\sim \begin{cases}   \text{Bern}({\mu_1(V_i, Z_i)}), & D^{\pi_0}_i=1 \\
         \text{Bern}({\mu_0(V_i, Z_i)}), & D^{\pi_0}_i=0 \\
    \end{cases}.
\end{align}

\textbf{Experimental Trials.} We randomly sample coefficient vectors $W_z$, $W_{\pi_0}$, $W_{\mu_1}$ and $W_{\mu_0}$ parameterizing $\pi_0$ and $\pi$ in each of the experiments reported in Section \ref{sec:numeric_experiments}. As a result, the oracle regret between the status quo and updated policy varies across experimental trials depending on the sigmoidal outcome probabilities induced by our randomly sampled coefficient vectors.

\newpage
\subsubsection{Assumption robustness experiments}

We now leverage the same synthetic setup outlined in Section \ref{sec:numeric_experiments} to stress-test the coverage of our regret intervals to assumption violations. Observe that the IV unconfoundedness condition is satisfied because $\gamma(X_i)$ does not depend on $U_i$, the relevance condition is satisfied when $\beta_0 > 0$, and the exclusion restriction is satisfied when $\beta_1 = 0$. We test robustness to violation of the IV relevance and exclusion assumptions by varying $\beta_0$, $\beta_1$, respectively. We test robustness to violation of the MSM by selecting values for $\Lambda^* \notin [\Lambda^{-1}, \Lambda]$.

In Figure \ref{fig:assmption-violation}, we provide bounds for policy performance measures as we introduce violations to the MSM and IV assumptions. The top row shows regions in which the MSM assumption $\Lambda=1.4$ is satisfied. The $\hat{R}_{\delta}$ interval yields valid coverage when $\Lambda^* \in [\Lambda^{-1}, \Lambda]$, but breaks when the assumption is violated. Observe that the relevance condition (middle row) controls the tightness of the regret interval, but does not introduce coverage violations. As a result, more heterogeneity across finite values of the instrument yields tighter uncertainty quantification, but does not impact coverage. This is inline with prior empirical evaluations conducted under the IV framework \citep{lakkaraju2017selective, kleinberg2018human}. The bottom row shows that violation of the exclusion violation does introduce coverage violations as $\beta_1$ increases. Across settings, observe the the baseline and delta regret intervals overlap for the PPV, which is inline with the results of our theoretical analysis.

\begin{figure*}[t!]
    \centering
    \includegraphics[width=\linewidth]{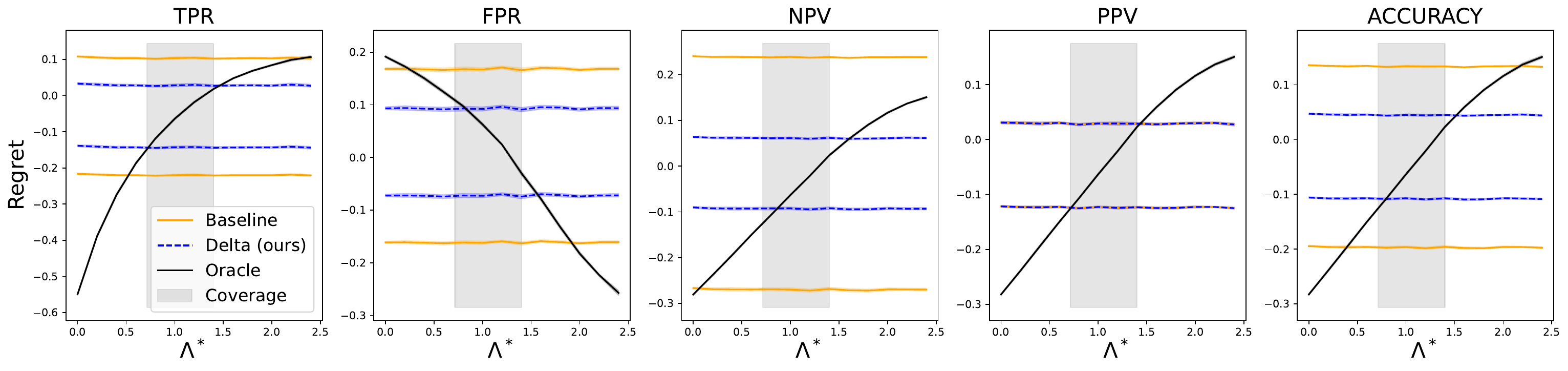}
    \centering
    \includegraphics[width=\linewidth]{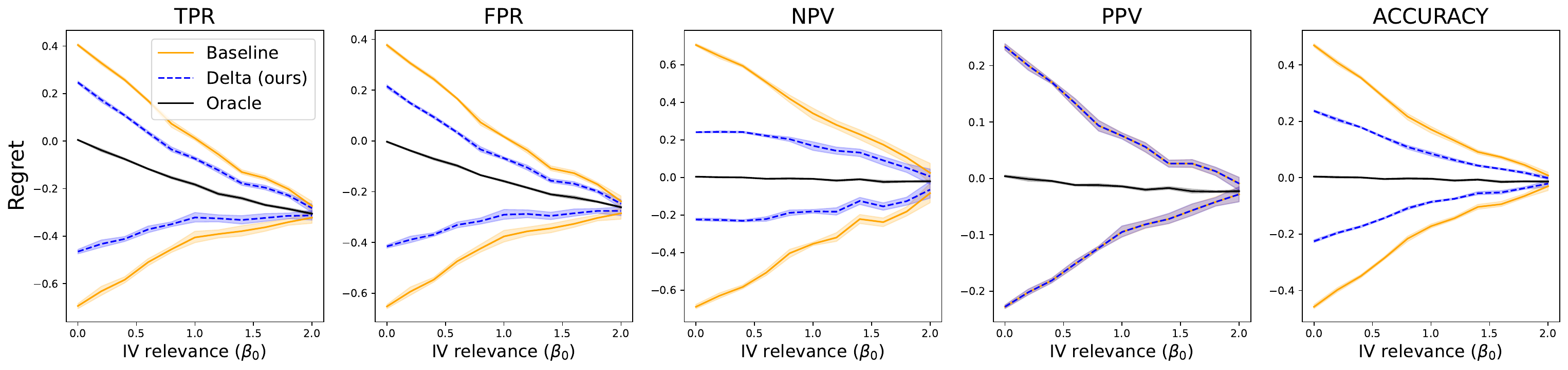}
     \centering
    \includegraphics[width=\linewidth]{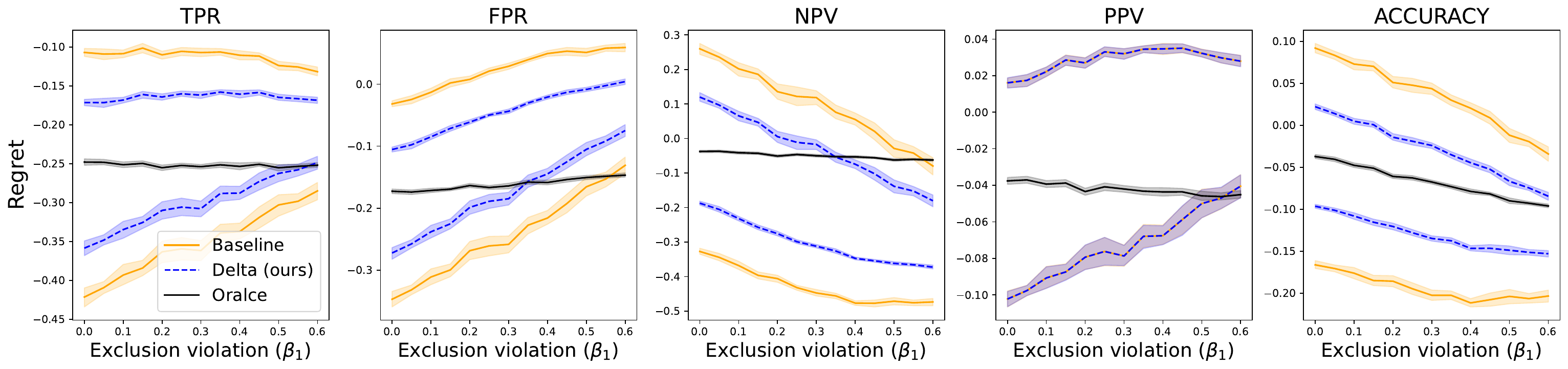}
    \caption{Top row (MSM): we fix $\Lambda=1.4$ and vary $\Lambda^*$. The area shown in grey indicates the values of $\Lambda^*$ for which the MSM assumption is satisfied. Middle row (IV): we vary $\beta_0$ controlling the relevance of $Z$ on the status quo decision-making policy. Bottom row (IV): we vary $\beta_1$ controlling the magnitude of the exclusion restriction violation. Shaded error region shows a 95\% bootstrap CI over $20$ trials.
    }\label{fig:assmption-violation}

\end{figure*}

\newpage
\subsubsection{Design sensitivity analysis}\label{appendix:realworld_setup}

\begin{figure*}[t!]
    \centering
    \includegraphics[width=\linewidth]{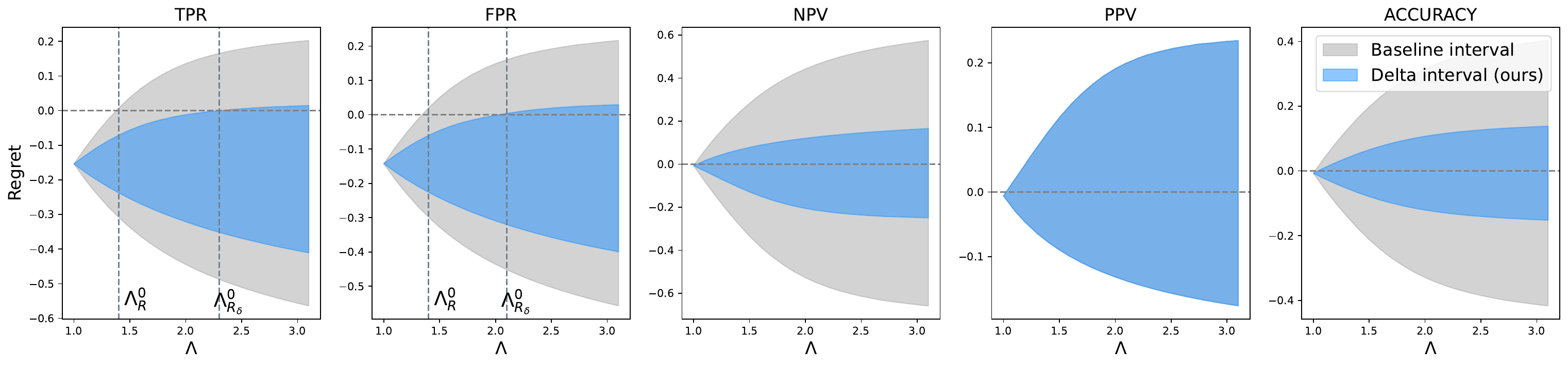}
    \vspace{-5mm}
    \caption{Design sensitivity analysis. Our $\delta$-regret interval certifies a policy performance difference up to a larger magnitude of unmeasured confounding $\Lambda^0$ for the TPR and FPR  measures.}\label{fig:design_sensitivity}
\end{figure*}

One common measure of robustness to confounding is the \textit{design sensitivity}, or the value of the sensitivity parameter at which the analysis exceeds a key threshold of interest \citep{rosenbaum1987sensitivity, rosenbaum2005sensitivity, namkoong2020off}. In our setting, we are most interested in the degree of unmeasured confounding permissible before the regret interval crosses zero. We conduct a design sensitivity analysis under the MSM by varying $\Lambda$ and measuring the value at which the $\delta$-regret and baseline intervals include zero. We denote these thresholds as $\Lambda_{R_\delta}^0$ and $\Lambda_R^0$, respectively. We leverage the same synthetic setup reported above and fit $\hat{\eta}$ via parametric logistic regression models to improve computation time. Figure \ref{fig:design_sensitivity} indicates that the improvement in our $\delta$ interval enables certifying a policy difference up to a larger magnitude of unmeasured confounding than would be possible via the baseline approach. For example, the $\delta$-interval certifies a TPR performance difference up to $\Lambda=2.3$, a significant improvement over $\Lambda=1.4$ yielded by the baseline interval.

\subsubsection{Real-world data experiments}\label{appendix:realworld_setup}

\textbf{Setup Details}. Each record of the synthetic dataset provided by \citep{obermeyer2019dissecting} contains patient demographics, health information measured in the prior year, and outcome variables measured in the following year. Only $\approx1\%$ of patients were enrolled in the care management program under the status quo policy. This low selection rate yields vacuous regret intervals without strong assumptions on confounding. Therefore, for this illustration, we construct a status quo policy with an $\approx 18\%$ selection rate ($N=2452$) by including all 452 records marked as enrolled in the program and randomly sampling 2000 unenrolled records. We fit the cost prediction model and nuisance functions using patient demographic variables, comorbidities, and prescription information measured in the preceding year. We use a linear regression to predict patient cost, and threshold predictions at the 55th percentile. This cutoff matches the threshold for physician enrollment recommendations of the deployed risk assessment \citep{obermeyer2019dissecting}. We use $40\%$ of the subsampled data to fit the cost prediction model and the remaining $60\%$ as an observational sample for the policy comparison task. Because we do not have access to physician identifiers which can be used as an instrument, we leverage the MSM assumption for identification. We use a logistic regression classifier to fit nuisance functions. We use Algorithm \ref{algorithm:plug-in} to estimate $[\underline{\hat{R}}_\delta(\pi, \pi_0; m, \mathcal{\hat{V}}), \hat{\overline{R}}_\delta(\pi, \pi_0; m, \mathcal{\hat{V}})]$ under the MSM with with $\Lambda=1.2$. We report results over $N=20$ trials with $K=2$ folds.

\newpage

\begin{figure*}
    \centering
    \includegraphics[width=1\linewidth]{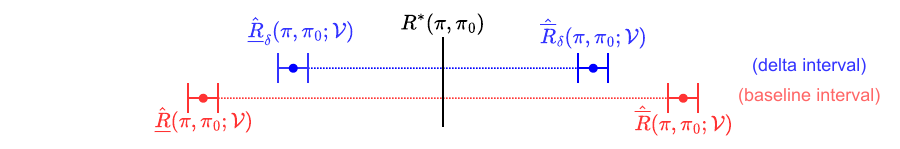}
    \caption{Two sources of uncertainty impacting policy regret bounds. (1) Confounding-related uncertainty impacts the width of the partially-identified regret interval (dotted lines). Our $\delta$-regret interval yields tighter partially-identified regret intervals than the baseline (Theorem \ref{thm:delta_seperation}). (2) Finite sample error impacts estimation of upper- and- lower interval end points, as shown by solid confidence intervals. Estimators for the $\delta$-regret interval and baseline interval have similar variance (see 95\% CI of each interval reported in Figure \ref{fig:pl_estimation}).}
    \label{fig:uncertainy-diagram}
\end{figure*}

\section{Characterizing Benefits of Uncertainty Cancellation for Off-Policy Evaluation}\label{appendix:uq_benefits}

In this Appendix, we provide additional discussion of our uncertainty cancellation approach. We begin by discussing the differential impact of \textbf{confounding-related uncertainty} and \textbf{finite sample uncertainty} on policy comparisons ($\S$ \ref{subsec:policy_uncertainty}). Our $\delta$-regret interval targets confounding-related uncertainty (Figure \ref{fig:uncertainy-diagram}). We then provide more intuition for our uncertainty cancellation approach ($\S$ \ref{subsec:intuiton}) and conclude by discussion further policy evaluation contexts which may benefit from our framework ($\S$ \ref{subsec:applications}).

\subsection{Sources of uncertainty impacting policy comparisons}\label{subsec:policy_uncertainty}

\textbf{Confounding-related uncertainty:} Let $\mu(x) = \bE[Y(1) \mid X=x]$ be the target outcome regression function. We can understand the impact of confounding related uncertainty on off-policy evaluation via the target regression decomposition 
$$
\mu(x) = \underbrace{\bE[Y(1) \mid D^{\pi_0}=1, X=x]}_{\text{Identified from observational data}} \cdot p(D^{\pi_0}=1 \mid X=x) +  \underbrace{\bE[Y(1) \mid D^{\pi_0}=0, X=x]}_{\in [0,1], \text{Unidentified from observational data}} \cdot p(D^{\pi_0}=0 \mid X=x). 
$$
Now suppose we are in the asymptotic setting with no finite sample uncertainty. Thus we know $e_d(x) = p(D^{\pi_0}=d \mid X=x)$ and the observed outcome regression 
$$
\mu_1(x) := \bE[Y(1) \mid D^{\pi_0}=1, X=x] = \bE[Y \mid D^{\pi_0}=1, X=x]
$$
from observational data.\footnote{The second equality follows by Assumption \ref{assumption:positivity}} However, because $D^{\pi_0} \nCI Y(1) \mid X$ due to confounding, the unobserved outcome regression $\mu_0(x) := \bE[Y(1) \mid D^{\pi_0}=0, X=x]$ is bounded within the worst-case interval $0 \leq \mu_0(x) \leq 1, \; \forall x \in X. $\newline

As we show in Figure \ref{fig:outcome_bounds}, bounds around the target regression $\mu(x) = e_1(x) \cdot \mu_1(x)  + (1-e_1(x)) \cdot \mu_0(x)$ widen as $e_1(x)$ decreases. This is because \textit{decreasing} the propensity score \textit{increases} the weighting of the unidentified term $\mu_0(x)$ in the regression decomposition. Therefore, the tightness of bounds around $\mu(x)$ depend on (1) the propensity function $e_d(x)$ and (2) the tightness of bounds around $\mu_0(x)$. The bounds around partially-identified $v$-statistics which we use to construct regret intervals inherits the same dependence on these terms (Lemma \ref{lemma:assumption_mapping}). As a result, \textbf{asymptotic regret intervals are wider when the status quo policy has a lower selection rate, regardless of the amount of data available for estimation.}

\textbf{Finite-sample uncertainty:} Finite sample uncertainty also impacts our ability to estimate regret interval endpoints. When less data is available, confidence intervals around upper and lower regret intervals will tend to be wider. As a result, although both the baseline and $\delta$ intervals have valid asymptotic coverage for the true regret (Appendix \ref{appendix:asymptotic_regret_identification}), coverage can be violated under bias in estimates of the regret interval end points. Figure \ref{fig:pl_estimation} shows confidence intervals around regret interval end-points as a function of dataset size. Our doubly-robust estimation approach improves data-efficiency of our regret estimator under no parametric assumptions on $\pi_0(x,u)$. We show that our DR estimator attains fast $\sqrt{n}$-rates in Appendix \ref{appendix:estimation_results}.

\begin{figure}[ht]
  \centering
    \includegraphics[width=.5\linewidth]{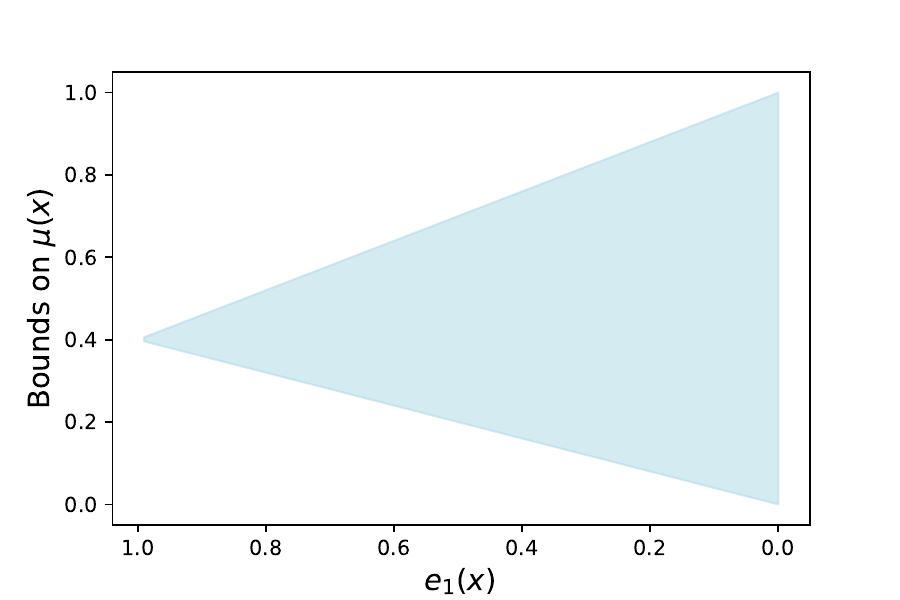}
     \vspace{-5mm}
    \caption{Bounds around $\mu(x) := \bE[Y(1)\mid X=x]$ widen as the propensity score (i.e., the selection rate under the status quo policy) $e_1(x)$ decreases. Recall that $\mu(x) = e_1(x) \cdot \mu_1(x)  + e_0(x) \cdot \mu_0(x)$. We fix $\mu_1(x) = .4$, let $\mu_0(x) \in [0,1]$, and let $e_1(x)$ vary from $0$ to $1$. Bounds around the target regression increase as the weighting on the unobserved regression $\mu_0(x)$ increases.}
    \label{fig:outcome_bounds}
\end{figure}

\newpage
\textbf{Synthetic data example.} We illustrate the relationship between confounding related uncertainty and finite sample error via a synthetic data experiment following a similar setup as the one outlined in Appendix \ref{appendix:experiments}. We construct subgroups by defining two protected attributes with two levels each. The first correlates with $X_1$, while the second correlates with $X_2$. This yields four intersectional subgroups G1-G4 with varying sizes ($\omega$) and selection rates under the status quo policy ($\gamma$).

Figure \ref{fig:fnr-synthetic-results} compares the $\delta$-regret interval and the baseline interval via worst case (WC) and instrumental variable (IV) partial identification strategies. We see that groups with lower selection rates under the status quo policy (e.g., G1, G2) have wider asymptotic bounds than those with higher selection rates (e.g., G3, G4). This indicates that the key driver of uncertainty at the asymptotic level is confounding. We can also examine statistical uncertainty in subgroup regret estimates by inspecting the size of confidence intervals. We observe that smaller subgroups (e.g., G1, G3) have larger variance in regret estimates than larger subgroups (e.g., G2, G4, full population).\footnote{We can isolate subgroup size ($\omega$) as the source of this uncertainty because we fix selection rates across subgroups such that G1($\gamma$)=G2($\gamma$), G3($\gamma$)=G4($\gamma$).}  Importantly, we observe that a small subgroup with a large selection rate (e.g., B5, $\gamma=.65$, $\omega=.08$) has tighter bounds than a larger subgroup with a lower selection rate (e.g., Population, $\gamma=0.5$, $\omega=1.0$). In line with our discussion above, this indicates that subgroup selection rates under $\pi_0$ are a key driver of uncertainty in asymptotic regret intervals.

\begin{figure}[ht]
  \centering
    \includegraphics[width=.9\linewidth]{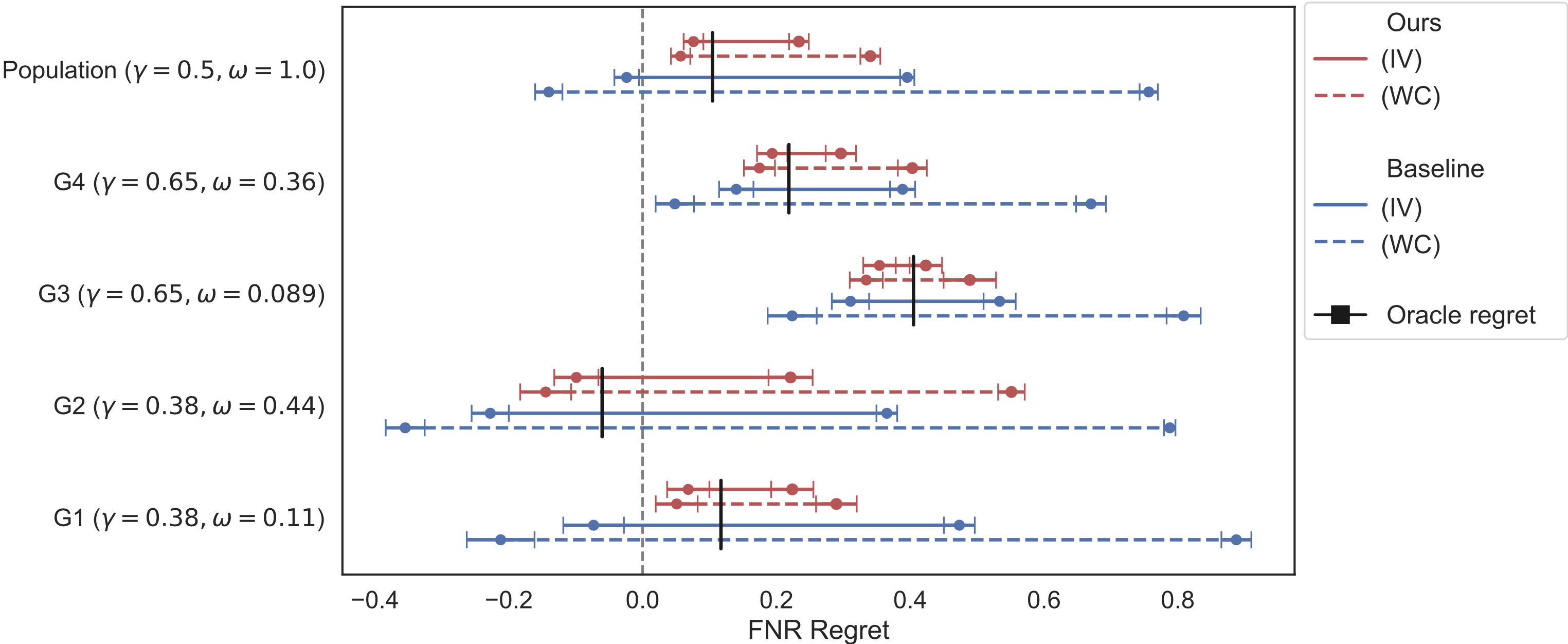}
    \caption{Top row indicates full population bounds, while lower rows show regret over subgroups of varying selection rates ($\gamma$) and sizes ($\omega$). Horizontal bars indicate asymptotic regret bounds under confounding. We show statistical uncertainty over $N=10$ runs by plotting 95\% confidence intervals centered at each upper and lower asymptotic bound.}
    \label{fig:fnr-synthetic-results}
\end{figure}
\newpage

\begin{figure}[ht]
  \centering
    \includegraphics[width=.4\linewidth]{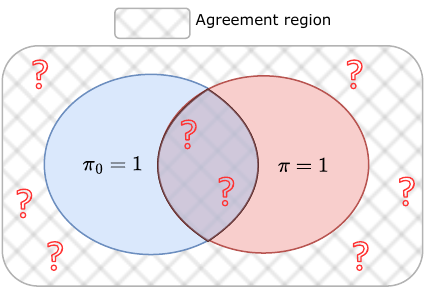}
    \caption{Uncertainty in the agreement region of the policy comparison (denoted via red question marks) cancels in comparative policy performance analyses.}
    \label{fig:uncertainty_cancellation}
\end{figure}

\subsection{Building intuition for uncertainty cancellation}\label{subsec:intuiton}

The intuition for our approach is that we can safely ignore redundant uncertainty in the agreement region of the policy comparison (Figure \ref{fig:uncertainty_cancellation}). We see this concretely via the accuracy regret decomposition
\begin{align*}
R^*(\pi, \pi_0; m_u) &= m^*_u(\pi) - m^*_u(\pi_0) \\
&= p(T^{\pi} = Y(1)) - p(D^{\pi_0} = Y(1)) \\
&= ( \mathbin{\textcolor{red}{v_1(1,1)}} + v_1(1,0) + v_0(0,1) + \mathbin{\textcolor{red}{v_0(0,0)}} )  - ( \mathbin{\textcolor{red}{v_1(1,1)}} + v_1(0,1) + v_0(1,0) + \mathbin{\textcolor{red}{v_0(0,0)}} ) \\
&= v_1(1,0) + v_0(0,1) - v_1(0,1) - v_0(1,0). \\
\end{align*}
Observe that the agreement terms \textcolor{red}{$v_1(1,1)$} and \textcolor{red}{$v_0(0,0)$} cancel when we take the difference across policies. The agreement region ($\pi_0=1, \pi=1$) at the center of the Venn Diagram \textit{does not} add uncertainty in our context because $Y(1)$ is observed when $\pi_0=1$. Therefore, cancellation of $v_1(1,1)$ does not improve regret bounds. However, the agreement region in the complement space ($\pi_0=0, \pi=0$) \textit{does} contribute to uncertainty because $Y(1)$ is unobserved when $\pi_0=0$. Therefore cancellation of $v_0(0,0)$ is the main driver of the performance improvement. Theorem \ref{thm:delta_seperation} formalizes this notion by showing that the improvement in the tightness of the $\delta$-regret interval is proportional to the amount of uncertainty in the \textcolor{red}{$v_0(0,0)$} term -- i.e., $\alpha = \overline{v}_0(0,0) - \underline{v}_0(0,0)$.

We next turn to the positive predictive value regret for an example where uncertainty cancellation \textit{does not} improve bounds
\begin{align*}
R^*(\pi, \pi_0; m_{a=1}) &= m^*_{a=1}(\pi) - m^*_{a=1}(\pi_0) \\
&= p(Y(1)=1 \mid T^{\pi} =1 ) - p(Y(1)=1 \mid D^{\pi_0}=1 ) \\
&= \frac{p(T=1, Y(1) = 1)}{p(T = 1)} - \frac{p(D=1, Y(1) = 1)}{p(D = 1)} \\
&= \frac{v_1(1,0) + \mathbin{\textcolor{red}{v_1(1,1)}}}{\rho_{11} + \rho_{10}} - \frac{\mathbin{\textcolor{red}{v_1(1,1)}} + v_1(0,1)}{ \rho_{11} + \rho_{01}}\\
\end{align*}
where $\rho_{td} = p(T=t, D=d)$. Because only \textcolor{red}{$v_1(1,1)$} cancels in this decomposition, we observe no improvement from the $\delta$-regret interval for this performance measure.

More generally, we will observe a performance improvement from the $\delta$-regret interval in any off-policy evaluation context where there is uncertainty arising from the agreement region of the policy comparison. \textbf{Importantly, this insight holds for any causal assumption (e.g., Rosenbaum's $\Gamma$, MSM, IV, Proximal) which can be expressed as pointwise bounding functions, which we detail in Appendix \ref{appendix:assumption_extensions}}. 

\subsection{Other potential applications of uncertainty cancellation}\label{subsec:applications}

While our work is concerned with policy evaluation under unmeasured confounding, other sources of uncertainty can also complicate policy comparisons. Our $\delta$-regret interval may support partial identification under these uncertainty sources.\footnote{We offer these examples as illustrative of our technique and acknowledge further technical development would be required to develop these approaches in practice.}

\textbf{Measurement Error and Noisy Labels.} In many cases, labels available for model evaluation are observed under measurement error or label noise \citep{scott2013classification, xia2019anchor, angluin1988learning, guerdan2023ground}. A common setup in this setting is to model the difference between the true outcome $Y^* \in \{0,1\}$ and its proxy $Y \in \{0,1\}$ via the false negative rate $\beta = p(Y=0 \mid Y^*=1)$ and false positive rate $\alpha = p(Y=1 \mid Y^*=0)$. However, when conducting comparative performance analyses, it may be possible to disregard measurement error in the agreement regions of the policy action space. This may tighten partial identification bounds studied in prior work \citep{fogliato2020fairness}.

\textbf{Missing Protected Attributes.} Protected attributes are sometimes unavailable for fairness assessments due to data collection or regulatory constraints \citep{kallus2022assessing, coston2019fair}. However, we may wish to compare fairness statistics of alternative policies under missing protected attributes. If a partial identification approach is used to bound fairness characteristics under missing protected attributes \citep{kallus2022assessing}, it may be possible to tighten bounds by studying fairness differences over the disagreement region of the policy space.


\end{document}